\documentclass[11pt, oneside]{article}   	
\usepackage{geometry}                		
\usepackage{graphicx}			
\usepackage{amsfonts,latexsym,amsthm,amssymb,amsmath,amscd,euscript}
\usepackage{bbm}
\usepackage{framed}
\usepackage{fullpage}
\usepackage{mathrsfs}
\usepackage{stmaryrd}
\usepackage{mathtools}
\usepackage{hyperref}
\usepackage{accents}
\usepackage{setspace}
\usepackage[usenames,dvipsnames]{xcolor}
\definecolor{plum}  {rgb}{.4,0,.4}
\definecolor{bred} {rgb}{0.6,0,0}
\hypersetup{colorlinks = true, linkcolor = bred, citecolor = blue, urlcolor=urlcolr}
\usepackage{float}
\usepackage[noadjust]{cite}
\usepackage{bm}
\usepackage{enumitem}
\usepackage{verbatim}
\usepackage[utf8]{inputenc}
\usepackage[T1]{fontenc}
\usepackage{microtype}
\usepackage[bottom]{footmisc}
\usepackage{tcolorbox}
\usepackage{tikz-cd,tikz,pgfplots,pgfplotstable}
\usetikzlibrary{matrix}
\usetikzlibrary{cd}
\usetikzlibrary{calc}
\usetikzlibrary{arrows}      
\usetikzlibrary{decorations.markings}
\usetikzlibrary{positioning}
\pgfplotsset{width=7cm,compat=1.8}  

\date{}

\makeatletter
\newtheorem*{rep@theorem}{\rep@title}
\newcommand{\newreptheorem}[2]{%
	\newenvironment{rep#1}[1]{%
		\def\rep@title{#2 \ref{##1}}%
		\begin{rep@theorem}}%
		{\end{rep@theorem}}}
\makeatother

\usepackage{thmtools,thm-restate}

\theoremstyle{plain}
\newtheorem{theorem}{Theorem}
\newtheorem{lemma}[theorem]{Lemma}

\newtheorem{proposition}[theorem]{Proposition}

\newtheorem{corollary}[theorem]{Corollary}
\newtheorem{definition}[theorem]{Definition}

\newtheorem{remark}[theorem]{Remark}

\theoremstyle{definition}

\newtheorem*{tldr*}{TL;DR}
\newtheorem*{impressions*}{Impressiones}
\newtheorem*{summary*}{Summary}

\newcommand{\E}{\mathbb{E}}
\newcommand{\ex}[2]{{\ifx&#1& \mathbb{E} \else \underset{#1}{\mathbb{E}} \fi \left[#2\right]}}
\newcommand{\var}[2]{{\ifx&#1& \mathsf{Var} \else \underset{#1}{\mathsf{Var}} \fi \left[#2\right]}}

\newcommand{\variance}{\mathrm{Var}}
\newcommand{\diag}{\mathrm{diag}}

\DeclareMathOperator{\tr}{Tr}
\newcommand{\sZ}{\tilde{Z}}

\newcommand{\Rb}{\mathbb{R}}
\usepackage{bbm}

\newcommand\independent{\protect\mathpalette{\protect\independenT}{\perp}}
\def\independenT#1#2{\mathrel{\rlap{$#1#2$}\mkern2mu{#1#2}}}

\newcommand{\prior}{Q}
\newcommand{\post}{{P}}
\newcommand{\CMI}[2]{{\ifx&#2& \mathsf{CMI} \else \mathsf{CMI}_{#2} \fi \left(#1\right)}}
\usepackage{etoolbox}
\newcommand{\define}[4]{\expandafter#1\csname#3#4\endcsname{#2{#4}}}
\forcsvlist{\define{\newcommand}{\mathcal}{c}}{K,C,F,H,S,X,Y,Z,E,D,N,L,G,P,Q,A,T,I,M,W,B,U,V,O}  

\DeclarePairedDelimiter\ceil{\lceil}{\rceil}

\newcommand{\changelocaltocdepth}[1]{%
  \addtocontents{toc}{\protect\setcounter{tocdepth}{#1}}%
  \setcounter{tocdepth}{#1}%
}
\let\originalleft\left
\let\originalright\right
\renewcommand{\left}{\mathopen{}\mathclose\bgroup\originalleft}
\renewcommand{\right}{\aftergroup\egroup\originalright}
\newcommand{\mybignote}[2]{}

\definecolor{lnkcolr}{rgb}{0.55, 0.3, 0.09} 
\definecolor{urlcolr}{rgb}{0.0, 0.35, 0.26}

\graphicspath{{./figures/}}

\begin{document}
\title{\textbf{Information Complexity and Generalization Bounds}} 
\author{Pradeep Kr. Banerjee\thanks{MPI MiS. \dotfill \texttt{pradeep@mis.mpg.de}} 
	\and Guido Mont{\'u}far\thanks{UCLA \& MPI MiS. \dotfill \texttt{montufar@math.ucla.edu}}}

\maketitle

\begin{abstract}
	We present a unifying picture of PAC-Bayesian and mutual information-based upper bounds on the generalization error of randomized learning algorithms. As we show, Tong Zhang's information exponential inequality (IEI) gives a general recipe for constructing bounds of both flavors. We show that several important results in the literature can be obtained as simple corollaries of the IEI under different assumptions on the loss function. Moreover, we obtain new bounds for data-dependent priors and unbounded loss functions. Optimizing the bounds gives rise to variants of the Gibbs algorithm, for which we discuss two practical examples for learning with neural networks, namely, Entropy- and PAC-Bayes- SGD. Further, we use an Occam factor argument to show a PAC-Bayesian bound that incorporates second-order curvature information of the training loss. 
\end{abstract}

{\hypersetup{linkcolor=bred} \tableofcontents}

\section{Introduction}
The generalization capability of a learning algorithm is intrinsically related to the information that the output hypothesis reveals about the input training dataset: The lesser the information revealed, the better the generalization. This argument has been formalized in recent years by appealing to different notions of information stability \cite{russoZou2016controlling,generalizationRaginsky,jiao2017dependence,generalizationRaginskyBassily,veeravalli,issa2019strengthened,steinke2020reasoning,steinke2020reasoningOpen,DziugaiteRoySGLDSteinke}. Information stability quantifies the sensitivity of a learning algorithm to local perturbations of its input, and draws on a rich tradition of earlier work on algorithmic \cite{bousquet2002stability,ShalevLearnabilityStabilityUnifConvergence,bassily2016algorithmic}, and distributional \cite{feldman2015generalization,rogers2016max,feldman2018calibrating} stability in adaptive data analysis. Closely related to the information stability approach is the so-called PAC-Bayesian approach to data-dependent generalization bounds, originally due to McAllester \cite{generalizationRaginskyPACBayesMcallester,mcallester1999some,mcallesterdropout}. While these two approaches have evolved independently of each other, a principal objective of this work is to present them under a unified framework.

We consider the standard apparatus of statistical learning theory \cite{shalev2014understanding}. We have an example domain $\cZ= \cX \times \cY$ of the instances and labels, a hypothesis space $\cW$, a fixed loss function  $\ell:{\cW}\times{\cZ}\to[0,\infty)$, and a training sample $S$, which is an $n$-tuple $(Z_1,\ldots,Z_n)$ of i.i.d.\ random elements of $\cZ$ drawn according to some unknown distribution $\mu$. A learning algorithm is a Markov kernel $P_{W|S}$ that maps input training samples $S$ to conditional distributions of hypotheses $W$ in $\cW$. This defines a joint distribution $P_{SW}=P_S P_{W|S}$, $P_S=\mu^{\otimes n}$, and a corresponding marginal distribution $P_W$. The \emph{true risk} of a hypothesis $w\in\mathcal{W}$ on $\mu$ is $L_{\mu}(w) := {\E}_{\mu}[\ell(w,Z)]$, and its \emph{empirical risk} on the training sample $S$ is $L_S(w):= \frac{1}{n} \sum_{i=1}^n \ell(w,Z_i)$. Our goal is to control the \emph{generalization error}, $\mathrm{g}(W,S) := L_{\mu}(W)-L_S(W)$, either in expectation, or with high probability. One difficulty in achieving this goal is the nontrivial statistical dependency between the sample $S$ and the learned hypothesis $W$.

For controlling the generalization error in expectation, we can rewrite the true risk of a given hypothesis $w$ as $L_{\mu}(w) = \E_{S'\sim \mu^{\otimes n}}[L_{S'}(w)]$, where $S'=({Z_1'},\ldots,{Z_n'})$ is an i.i.d.\ sample. Then the expected generalization error can be written as a difference of two expectations of the same loss function, 
\begin{align*}
&\E_{SW}[\mathrm{g}(W,S)]=\E_{P_S\otimes P_W}[L_{S}(W)]- \E_{P_{SW}}[L_S(W)],
\end{align*}
where the second expectation is taken w.r.t.\ the joint distribution of the training sample and the output hypothesis, while the first expectation is taken w.r.t.\ the product of the two marginal distributions. Hence the expected generalization error reflects the dependence of the output $W$ on the input $S$. This dependence can also be measured by their mutual information as has been shown in recent works  \cite{generalizationRaginsky,jiao2017dependence,generalizationRaginskyBassily,veeravalli,issa2019strengthened,steinke2020reasoning}. We refer to such bounds as mutual information-based generalization bounds.

Alternatively, we may wish to control the generalization error of the learning algorithm $P_{W|S}$ with high probability over the training sample $S$. The expected generalization error over hypotheses chosen from the distribution $\post$ (\emph{posterior}) output by the learning algorithm, i.e.,\ $\E_{\post}[\mathrm{g}(W,S)]$, can be upper-bounded with high probability under $P_S$ by the KL divergence between $\post$ and an arbitrary reference distribution $\prior$ (\emph{prior}), that is selected \emph{before} the draw of the training sample $S$. For any $\prior$, these bounds hold {uniformly} for all $\post$, and are called PAC-Bayesian bounds \cite{generalizationRaginskyPACBayesMcallester,mcallester1999some,mcallesterdropout,TongZhangPACBayes2006,catonibook,maurer2004PACBayes,TimvanErvenPACBayesTongZhang,grunwaldTongZhangFastrates,alquierPACBayesPropertiesGibbs,germain2009Paclinear,FlatminimaBayesianSGDGermainMackay1992,rivasplataThiemannPACBackprop,rivasplatapac}, where PAC stands for \emph{Probably Approximately Correct}. Bounds of this type are useful when we have a fixed dataset $s\in\cZ^n$ and a new hypothesis is sampled from $\post$ every time the algorithm is used. Choosing the posterior to minimize a PAC-Bayesian bound leads to the well-known Gibbs-ERM principle \cite{TongZhangPACBayes2006,generalizationRaginsky,alquierPACBayesPropertiesGibbs,rivasplatapackuzborskij}. On the other hand, for a fixed posterior $\post$, $\E_S[D(\post\|\prior)]$ is minimized by the \emph{oracle prior}, $\prior^{\star}=\E_{S}[P_{W|S}(\cdot|S)]$. Note $\E_S[D(\post\|\prior^{\star})]$ is just the mutual information $I(S;W)$, which is the key quantity controlling the expected generalization error in \cite{generalizationRaginsky,jiao2017dependence,generalizationRaginskyBassily}. 

\paragraph{Summary of contributions.}
We present a unified framework for deriving PAC-Bayesian and mutual information-based generalization bounds, starting from a fundamental information-theoretic inequality, Lemma~\ref{lem:InfoExpInequality}, due to Tong Zhang \cite{TongZhangPACBayes2006}. Besides recovering several well-known bounds of both flavors, such as the Xu-Raginsky mutual information-bound \cite{generalizationRaginsky} in Corollary~\ref{cor:XuRaginsky}, and Catoni's bound \cite{catonibook} in Corollary~\ref{thm:pacbayesCatoniDiff}, we also obtain new bounds for data-dependent priors (Proposition \ref{thm:PACDPPrior}) and unbounded loss functions (Theorem~\ref{cor:TongZhangIRM} and Proposition \ref{thm:PACUnionboundMIsubGaussian}). Proposition~\ref{thm:CMIPACBound01Loss} gives a PAC-Bayes version of the CMI bound due to Steinke and Zakynthinou \cite{steinke2020reasoning}. Optimizing these bounds w.r.t.\ the posterior gives rise to variants of the Gibbs algorithm, for which we discuss two examples and show how Catoni's bound can be used to derive a PAC-Bayes-SGD \cite{dziugaitenonvacuous} objective. In Proposition~\ref{thm:OccamPAC}, we give a PAC-Bayesian bound motivated by an Occam's factor argument, in relation to ``flat'' minima in neural networks \cite{FlatMinimaSchmidhuber}.

\section{Preliminaries}
We write $\cM(\cW)$ to denote the family of probability measures over a set $\cW$, and $\cK(\cS,\cW)$ to denote the set of Markov kernels from $\cS$ to $\cW$. 
Proposition~\ref{Prop:logmgfProperties} collects some well-known facts about the cumulant generating function $\Lambda_X(\beta)=\ln\E[e^{\beta X}]$ of a random variable $X$ for $\beta > 0$ (see, e.g., \cite[\S 2]{boucheron2013concentration}, and \cite{TongZhangPACBayes2006}):
\begin{proposition}[Facts about the cumulant generating function $\Lambda_X(\beta)$ for $\beta>0$]\label{Prop:logmgfProperties}\
	
	\begin{enumerate} 
		\item $\Lambda_X(\beta)$ is infinitely differentiable and convex in $\beta$;
		\item $\tfrac{1}{\beta}\Lambda_X(\beta)$ is an increasing function of $\beta$;
		\item $\E[X]\le \tfrac{1}{\beta}\Lambda_{X}(\beta)\le \Lambda'_{X}(\beta)$;
		\item For real constants $a,b$, $\tfrac{1}{\beta}\Lambda_{aX+b}(\beta)=\tfrac{1}{\beta}\Lambda_X(a\beta)+b$;
		\item $\tfrac{1}{\beta}\Lambda_X(\beta) \le  \E[X]+\frac{\beta}{2}\variance(X)+O(\beta^2)$, where $\variance(X) = \E\left[(X-\E[X])^2\right]$;
		\item If $X\in[0,1]$, then $\tfrac{1}{\beta}\Lambda_X(\beta)\le \tfrac{1}{\beta}\ln \left(1-(1-e^{\beta})\E[X]\right)$, with equality when $X\in\{0,1\}$ is Bernoulli;
		\item $X$ is \emph{$\sigma$-sub-Gaussian} if 
		$\tfrac{1}{\beta}\Lambda_{X}(\beta)\le \E[X]+\frac{\beta\sigma^2}{2}$;
		\item $X$ is \emph{$(\sigma,c)$-sub-gamma} if $\tfrac{1}{\beta}\Lambda_{X}(\beta)\le \E[X]+ \frac{\beta\sigma^2}{2(1-c\beta)}$ for every $\beta$ such that $\beta \in \left(0,\tfrac{1}{c}\right)$.
	\end{enumerate}
\end{proposition}
We also note that $\Lambda_{X}(0)=0,\,\Lambda'_{X}(0)=\E[X]$.

We will need the following characterization of the inverse of the Fenchel-Legendre dual of a smooth convex function:
\begin{restatable}[{\cite[Lemma 2.4]{boucheron2013concentration}}]
	{lemma}{InverseFenchelConjugate}\label{lemma:psistar}
	Let $\psi$ be a convex and continuously differentiable function defined on the interval $[0,b)$, where $0<b\le \infty$. Assume that $\psi(0)= \psi'(0) = 0$. 	Then, the Legendre dual of $\psi$, defined as $$\psi^*(t) := \sup_{\beta \in [0,b)} \{\beta t-\psi(\beta)\},$$ is a nonnegative convex and nondecreasing function on $[0,\infty)$ with $\psi^*(0) = 0$. Moreover, for every $y\ge 0$, the set $\{t\ge 0: \psi^*(t)> y\}$ is non-empty and the generalized inverse of $\psi^*$ defined by $\psi^{*-1}(y) = \inf\{t\ge 0: \psi^*(t)> y\}$ can also be written as 
	\begin{align*}
	\psi^{*-1}(y) = \inf_{\beta \in (0,b)} \frac{y+\psi(\beta)}{\beta}.
	\end{align*}
\end{restatable}
We will need the following property of the Gibbs measure:
\begin{lemma}
	[{\cite[Proposition 3.1]{TongZhangPACBayes2006}}]
	\label{thm:DVlemmalikeCatoni}
	For any real-valued measurable function $f$ on $\cW$, any real $\beta>0$, and any $\post,\,\prior\in\cM(\cW)$ such that $D(\post\|\prior)<\infty$, we have $$\beta^{-1}D(\post\|\post^{\ast})=\E_{\post}[f(W)]+\beta^{-1}D(\post\|\prior)+\beta^{-1}\ln \E_{\prior} [e^{-\beta f(W)}],$$
	where $\post^{\ast}$ is the Gibbs measure $$\post^{\ast}(\mathrm{d}w):= \frac{e^{-\beta f(w)}}{\E_{\prior}[e^{-\beta f(W')}]}\prior(\mathrm{d}w).$$ Consequently,
	\begin{align*}
	\inf_{\post\in\cM(\cW)}\left\{\E_{\post}f(W)+\beta^{-1}D(\post\|\prior)\right\}=-\beta^{-1}\ln \E_{\prior} [e^{-\beta f(W)}].
	\end{align*} 
\end{lemma}
Finally, we recall the \emph{golden formula}: For all $Q\in\cM(\cW)$ such that $D(P_W\|Q) < \infty$, we have 
\begin{align}\label{eq:goldenformulaMI}
I(S;W )= D(P_{W|S}\|Q|P_S) - D(P_W\|Q),
\end{align}
where $D(P_{W|S} \|Q  |P_{S})=\int_{\cZ^{n}} D(P_{W|S=s} \| Q) \mu^{\otimes n}(\mathrm{d} s)$.

All information-theoretic quantities are expressed in \emph{nats}, unless specified otherwise. 
All proofs are relegated to Appendix \ref{sec:proofs}.

\section{One bound to rule them all}\label{sec:OneBoundAll}
\subsection{The information exponential inequality}\label{sec:IEI}
For any real $\beta>0$, define
\begin{align}\label{eq:TongZhanglogmgf}
M_{\beta}(w) = -\beta^{-1}\Lambda_{-\ell(w,Z)}(\beta)=-\beta^{-1}\ln \E_{\mu}[e^{-\beta \ell(w,Z)}],
\end{align}
which acts as a surrogate for $L_{\mu}(w)$. Following \cite{grunwaldTongZhangFastrates}, we call this quantity the \emph{annealed expectation}. 
\begin{restatable}[{Information exponential inequality (IEI) \cite[Lemma 2.1]{TongZhangPACBayes2006}}]{lemma}{InfoExpInequality}\label{lem:InfoExpInequality}
	For any prior $\prior\in \cM(\cW)$, any real-valued loss function $\ell$ on $\cW\times\cZ$, and any posterior distribution $\post\ll\prior$ over $\cW$ that depends on an i.i.d.\ training sample $S$, we have $${\E}_{S}\exp \big\{n\beta{\E}_{\post}\left[M_{\beta}(W)-L_S(W)\right]-D(\post \| \prior)\big\} \le 1.$$
\end{restatable}
The IEI implies bounds both in probability and in expectation for the quantity $$n\beta{\E}_{\post}\left[M_{\beta}(W)-L_S(W)\right]-D(\post \| \prior),$$ and is the key tool for showing the following theorem due to Tong Zhang that holds for unbounded loss functions: 
\begin{restatable}[{\cite[Theorem 2.1]{TongZhangPACBayes2006}}]{theorem}{TongZhang}\label{thm:TongZhangIRM}
	Let $\mu$ be a distribution over $\cZ$, and let $S$ be an i.i.d.\ training sample from $\mu$. Let $\prior\in\cM(\cW)$ be a prior distribution that does not depend on $S$, and let $\ell$ be a real-valued loss function on $\cW\times\cZ$. Let $\beta >0$, and let $\delta\in (0,1]$. Then, with probability of at least $1-\delta$ over the choice of $S \sim\mu^{\otimes{n}}$, for all distributions $\post \ll \prior$ over $\cW$ (even such that depend on $S$), we have: 
	\begin{align}\label{eq:TongZhangHighProbBound}
	\E_{\post} [M_{\beta}(W)] \leq \E_{\post}[L_S(W)] + \frac{1}{n\beta}\left(D(\post\|\prior)+\ln \frac{1}{\delta}\right).
	\end{align}
	Moreover, we have the following bound in expectation:
	\begin{align}\label{eq:TongZhangExpBound}
	\E_{SW} [M_{\beta}(W)] \leq \E_{SW}[L_S(W)] + \frac{1}{n\beta }D(\post\|\prior|P_S).
	\end{align}
\end{restatable}
Following \cite{TongZhangPACBayes2006AnnalsStats}, we call the regularized empirical risk $$\E_{\post}[L_S(W)] + \frac{1}{n\beta}D(\post\|\prior)$$ as the \emph{Information Complexity (IC)}, which is a data- and algorithm- dependent quantity. 

It is useful to replace the annealed expectation $M_{\beta}(w)$ in \eqref{eq:TongZhangHighProbBound} and \eqref{eq:TongZhangExpBound} with the true risk $L_{\mu}(w)$.
By Proposition \ref{Prop:logmgfProperties} items \emph{1)} and \emph{3)}, we have $M_{\beta}(w) \le L_{\mu}(w)$. 
For general loss functions, Proposition \ref{Prop:logmgfProperties} item \emph{5)} is useful for getting bounds in the opposite direction. By items \emph{4)}, \emph{7)} and \emph{8)} of Proposition~\ref{Prop:logmgfProperties}, if for all $w\in\cW$, $\ell(w,Z)$ is $\sigma$-sub-Gaussian, resp., $(\sigma,c)$-sub-gamma under $\mu$, then we have for all $w\in\cW$, $L_{\mu}(w) \le M_{\beta}(w) + \frac{\beta}{2}\sigma^2$ for all $\beta>0$, resp., $L_{\mu}(w) \le M_{\beta}(w) + \frac{\beta}{2(1-c\beta)}\sigma^2$ for every $\beta$ such that $\beta \in \left(0,\tfrac{1}{c}\right)$.
More generally, we note the following result, which follows as a corollary to Theorem~\ref{thm:TongZhangIRM} and Lemma~\ref{lemma:psistar}:
\begin{restatable}{theorem}{TongZhangCor}
	\label{cor:TongZhangIRM}
	Suppose that there exist a convex function $\psi\colon \mathbb{R}_{\geq0} \rightarrow \Rb$ satisfying $\psi(0)=\psi^{\prime}(0)=0$, such that 
	\begin{align}\label{eq:AssumptionStar}
	\sup_{w \in \cW} \left[L_{\mu}(w)-M_{\beta}(w)\right]\leq \frac{\psi(\beta)}{\beta},\, \forall\beta>0. 
	\end{align}
	Then, under the setting of Theorem~\ref{thm:TongZhangIRM},
	with probability of at least $1-\delta$ over the choice of $S \sim\mu^{\otimes{n}}$, for all distributions $\post \ll \prior$ over $\cW$ (even such that depend on $S$), we have 
	\begin{align}\label{eq:TongZhangHighProbBoundOpt}
	\E_{\post} [\mathrm{g}(W,S)] &\le \frac{1}{n\beta }\left(D(\post\|\prior)+\ln \frac{1}{\delta}\right)+\frac{\psi(\beta)}{\beta}.
	\end{align}
	Moreover, we have the following bound in expectation: 
	\begin{align}\label{eq:TongZhangExpBoundOpt}
	\E_{SW} [\mathrm{g}(W,S)] &\le \psi^{*-1}\left(\frac{D(\post\|\prior|P_S)}{n}\right).
	\end{align}
\end{restatable}
By the golden formula \eqref{eq:goldenformulaMI}, 
under the oracle prior $\prior^{\star}=\E_S[P_{W|S}]$, $\E_S[D(\post\|\prior^{\star})]=I(S;W)$. 
If $\ell(w,Z)$ is $\sigma$-sub-Gaussian under $\mu$ for all $w\in\cW$, then we can take $\psi(\beta)=\tfrac{\beta^2\sigma^2}{2}$ for every $\beta > 0$ and $\psi^{*-1}(y)=\sqrt{2\sigma^2y}$ \cite[\S 2.3]{boucheron2013concentration},
in which case we recover the bound in expectation due to Xu and Raginsky \cite{generalizationRaginsky}: 
\begin{corollary}\label{cor:XuRaginsky}
	If $\ell(w,Z)$ is $\sigma$-sub-Gaussian under $\mu$ for all $w\in\cW$, then $$\E_{SW} [\mathrm{g}(W,S)] \le \sqrt{2\sigma^2 I(S;W)/n}.$$
\end{corollary}

Corollary \ref{cor:XuRaginsky} shows that an algorithm that reveals a small amount of information about its input generalizes well. 
This observation, for instance, forms the basis for the Gibbs algorithm, which can be thought of as ``stabilizing'' the empirical risk minimization (ERM) algorithm by controlling the input-output mutual information $I(S;W)$ \cite{generalizationRaginsky}.
We discuss extensions of this idea in Section \ref{sec:ICM}. In Appendix \ref{sec:SFRL}, we highlight a functional characterization of the mutual information in relation to the ``single-draw'' generalization bound due to \cite{generalizationRaginskyBassily}.

For a $(\sigma,c)$-sub-gamma under $\mu$, we can take $\psi(\beta)=\tfrac{\beta^2\sigma^2}{2(1-\beta c)}$ for every $\beta$ such that $\beta \in \left(0,\tfrac{1}{c}\right)$ and $\psi^{*-1}(y)=\sqrt{2\sigma^2y}+cy$ \cite[\S 2.4]{boucheron2013concentration}, which gives the following result:
\begin{corollary}\label{cor:subgammaMI}
	If $\ell(w,Z)$ is $(\sigma,c)$-sub-gamma under $\mu$ for all $w\in\cW$, then $$\E_{SW} [\mathrm{g}(W,S)] \le \sqrt{2\sigma^2 I(S;W)/n}+cI(S;W)/n.$$
\end{corollary}
Fixing $\beta=1$ in \eqref{eq:TongZhangHighProbBoundOpt}, we recover \cite[Corollary 5]{FlatminimaBayesianSGDGermainMackay1992}:
\begin{corollary}\label{thm:PACBayesSubGamma}
	Consider the setting in Theorem~\ref{thm:TongZhangIRM}. If the loss $\ell$ is $(\sigma,c)$-sub-gamma with $c < 1$, then with probability of at least $1-\delta$ over the choice of $S \sim\mu^{\otimes{n}}$, for all distributions $\post \ll \prior$ over $\cW$, $$\E_{\post} [\mathrm{g}(W,S)]\le \frac{1}{n}\left(D(\post\|\prior)+ \ln\frac{1}{\delta}\right) + \frac{\sigma^2}{2(1-c)}.$$
\end{corollary}
The condition $c < 1$ guarantees that $\beta =1\in \left(0,\tfrac{1}{c}\right)$ when the sub-gamma condition in Proposition~\ref{Prop:logmgfProperties} item \emph{8)} is satisfied. 
In the limit $c\to 0_+$, a sub-gamma loss reduces to the sub-Gaussian loss \cite[\S 2.4]{boucheron2013concentration}, and we recover \cite[Corollary 4]{FlatminimaBayesianSGDGermainMackay1992}.

For the sub-Gaussian loss, fixing $\beta=1/\sqrt{n}$ in \eqref{eq:TongZhangHighProbBoundOpt}, the second term ${\psi(\beta)}/{\beta}$ decays with increasing $n$, but then the first term will have a slower decay of $1/\sqrt{n}$ instead of $1/n$.

We can also optimize $\beta$ in \eqref{eq:TongZhangHighProbBoundOpt} at a small cost using the union bound:
\begin{restatable}{proposition}{PACUnionBound}\label{thm:PACUnionboundMIsubGaussian}
	Consider the setting in Theorem~\ref{thm:TongZhangIRM}. If $\ell(w,Z)$ is $\sigma$-sub-Gaussian under $\mu$ for all $w\in\cW$, then for any constants $\alpha>1$ and $v>0$, and any $\delta\in (0,1]$, for all $\beta\in(0,v]$, with probability of at least $1-\delta$, we have
	\begin{align*}
	\E_{\post} [\mathrm{g}(W,S)]\le  
	\frac{\alpha}{n\beta}\left(D(\post\|\prior)+\ln \frac{\log_{\alpha}\sqrt{n}+K}{\delta}\right)+\frac{\beta\sigma^2}{2},
	\end{align*}
	where $K=\max\{\log_{\alpha}\big(\frac{v\sigma}{\sqrt{2\alpha}}\big),0\}+e$.
\end{restatable}

\subsection{The conditional mutual information (CMI) bound} \label{sec:CMIb}
One drawback of the mutual information-based bounds in Corollaries~\ref{cor:XuRaginsky} and \ref{cor:subgammaMI} is that $I(S;W)$ can be unbounded in many practical situations of interest \cite{veeravalli,steinke2020reasoning}. CMI-based bounds \cite{steinke2020reasoning,steinke2020reasoningOpen} address this issue by  conditioning on a superset of the training sample called the \emph{supersample}, in effect, normalizing the information content of each datum to one bit. 
As nicely articulated by Steinke and Zakynthinou \cite{steinke2020reasoning}, intuitively, the difference between the CMI- and MI- based approaches is that between ``recognizing'' vs. ``reconstructing'' the input, given the output of the algorithm. Recognizing the input is formalized by considering a i.i.d.\ supersample $\tilde{Z}\in\cZ^{n\times 2}$ consisting of $n\times 2$ data points, which comprises of $n$ ``true'' input data points mixed with $n$ ``ghost'' data points. A selector variable $U\in \{0,1\}^n$ chooses the input samples from the supersample, uniformly at random. Given the output of the algorithm, CMI then measures how well it is possible to distinguish the true inputs from their ghosts. We note the following definition:
\begin{definition}[{CMI of an algorithm $P_{W|S}$  \cite{steinke2020reasoning}}]\label{def:CMI}
	Let $\mu$ be a probability distribution on $\cZ$ and let $\sZ\in\cZ^{n\times 2}$ consist of $2n$ i.i.d.\ samples drawn from $\mu$. Let $U=(U_1,\ldots,U_n)\in \{0,1\}^n$ be uniformly random and independent from $\sZ$ and the randomness of the algorithm. Define $S:=\sZ_U \in \cZ^n$ by $(\sZ_U)_i = \sZ_{i,U_i+1}$ for all $i \in [n]$, i.e., $S$ is the subset of $\sZ$ indexed by $U$. Then the conditional mutual information (CMI) of an algorithm $P_{W|S}$ w.r.t.\ $\mu$ is  
	\begin{align*}
	\CMI{P_{W|S}}{\mu}:=I(W;U|\sZ).
	\end{align*}
\end{definition}
Since $S$ is a deterministic function of $\sZ U$, we have $W-\sZ U - S$. Also $W-S-\sZ U$ since $W$ depends on $\sZ U$ only through $S$. Together, this implies $$I(S;W)=I(\sZ U;W)=I(W;\sZ)+I(W;U|\sZ).$$
Suppose that we observe the output $W$ and wish to identify $S$ given access to $\sZ$. For any estimator $\widehat{U} = \xi(W,\sZ)$ of $U$, by Fano's inequality we have
\begin{align*}
	\inf_{\xi}\Pr\big(\xi(W,\sZ) \neq U\big) \ge 1-\frac{{I(W;U|\sZ)}+\log 2}{n\log 2}.
\end{align*}
$I(W;U|\sZ)$ thus upper-bounds the probability of successfully identifying $U$ from $\widehat{U}$.

In \cite[Theorem 2(1)]{steinke2020reasoning}, it is shown that for a $[0,1]$-valued loss, $$\E_{SW} [\mathrm{g}(W,S)]\le \sqrt{2 \cdot \CMI{P_{W|S}}{\mu}/n}.$$
Unlike the mutual information $I(S;W)$ that can be potentially unbounded, $\CMI{P_{W|S}}{\mu}$ is bounded above by $n\log 2$. 

We give a PAC-Bayesian version of the CMI bound in Proposition~\ref{thm:CMIPACBound01Loss}. 
Let $\bar{U}=(\bar{U}_1,\ldots,\bar{U}_n)$ be a vector obtained by inverting all the bits of $U$, and define $\bar{S}=\sZ_{\bar{U}}$. $S$ and $\bar{S}$ have a common marginal distribution, $\mu^{\otimes n}$. The algorithm maps the input $S=\sZ_U$ to a random element $W$ of $\cW$. Since $\bar{S}\independent W$, we can define the generalization error as
$\mathrm{g}(W,{\sZ,U}):=L_{\bar{S}}(W)-L_S(W)$, where $L_{\bar{S}}(w):= \frac{1}{n} \sum_{i=1}^n \ell(w,(\sZ_{\bar{U}})_i)$, and $L_S(w):= \frac{1}{n} \sum_{i=1}^n \ell(w,(\sZ_U)_i)$. 
Given a realization of the supersample $\sZ = \tilde{z}$ and selector variable $U=u$, we write $Q\equiv Q_{W|\tilde{z}}$ 
and 
$P\equiv P_{W|{\tilde{z},u}}$
for, resp., the prior and the posterior distribution. 
Then the following bounds hold 
for all such prior and posterior distributions: 
\begin{restatable}{proposition}{CMIPACBound}\label{thm:CMIPACBound01Loss}
	For any $[0,1]$-valued loss function $\ell$, for any $\beta>0$ and $\delta\in (0,1]$, with probability of at least $1-\delta$ 
	over a draw of {$\sZ,U$} as defined above, 
	we have:
	\begin{align}\label{eq:CMIPACBound01Loss}
	\E_{\post}[\mathrm{g}(W,{\sZ,U})]
	&\le \frac{1}{n\beta }\left(D(\post\|\prior)+\ln \frac{1}{\delta}\right)+\frac{\beta}{2}.
	\end{align}
	Moreover, we have the following bound in expectation: 
	\begin{align}\label{eq:CMIPACBound01LossExpectation}
	\E_{W,{\sZ,U}} [\mathrm{g}(W,{\sZ,U})] \le \sqrt{\frac{2\cdot D(\post\|\prior|P_{{\sZ,U}})}{n}}.
	\end{align}
\end{restatable}
Using the same reasoning as earlier, supplanting the associated oracle prior recovers the bound in expectation $\sqrt{2 \cdot \CMI{P_{W|S}}{\mu}/n}$ in \cite[Theorem 2(1)]{steinke2020reasoning}. 

\subsection{Recovering classical PAC-Bayesian bounds} 
\label{sec:CatoniMcAllRecovery}
By Proposition \ref{Prop:logmgfProperties} item \emph{6)}, for a $\{0,1\}$-valued loss, we have $$M_{\beta}(w) = -{\beta}^{-1} \ln \left(1-(1-e^{-\beta})L_{\mu}(w)\right)=:\Phi_{\beta}(L_{\mu}(w)).$$
$\Phi_{\beta}$ is an increasing one-to-one mapping of the unit interval onto itself, and is convex for $\beta > 0$.
The inverse of $\Phi_{\beta}$ is given by $\Phi_{\beta}^{-1}(x)=\tfrac{1-e^{-\beta x}}{1-e^{-\beta}}$, and we recover Catoni's PAC-Bayesian bound:
\begin{restatable}[{Catoni's bound \cite[Theorem 1.2.6]{catonibook}}]{corollary}{catoniD}\label{thm:pacbayesCatoniDiff}
	For any $\{0,1\}$-valued loss $\ell$, any distribution $\mu$, prior $\prior\in\cM(\cW)$, any real $\beta >0$, and any $\delta \in (0,1]$, with probability of at least $1 -\delta$ over $S\sim \mu^{\otimes n}$, 
	we have 
	for all $\post\ll\prior$ over $\cW$:
	\begin{align*}
	\E_{\post}[L_{\mu}(W)] \le \Phi_{\beta}^{-1}\left\{\E_{\post}[L_S(W)]+ \frac{1}{n\beta} \left(D(\post\|\prior) +\ln\frac{1}{\delta}\right)\right\}.
	\end{align*}
\end{restatable}

Using $1\le {\beta}{(1-e^{-\beta})}^{-1} \le {(1-\tfrac{\beta}{2})}^{-1}$, we have
\begin{align}
\E_{\post}[L_{\mu}(W)] &\le \Phi_{\beta}^{-1}\left\{
\E_{\post}[L_S(W)] + \frac{1}{n\beta} \left( D(\post\|\prior) +\ln\frac{1}{\delta}\right)\right\}\notag\\
&\le \frac{\beta}{1-e^{-\beta}}\left[ \E_{\post}[L_S(W)]+\frac{1}{n\beta}\left(D(\post\|\prior)+\ln \frac{1}{\delta}\right)\right] \label{eq:Catonibound1}\\
&\le \frac{1}{1-\tfrac{\beta}{2}}\left[ \E_{\post}[L_S(W)]+\frac{1}{n\beta}\left(D(\post\|\prior)+\ln \frac{1}{\delta}\right)\right]. \label{eq:McAllesterLinear}
\end{align}
\eqref{eq:Catonibound1} and \eqref{eq:McAllesterLinear} recover, resp., Catoni's \cite[Theorem 1.2.1]{catonibook} and McAllester's ``Linear PAC-Bayes bound'' \cite[Theorem 2]{mcallesterdropout}, where for the latter we additionally require that $\beta < 2$. 

For loss functions bounded in $[0,1]$, we elaborate in Appendix~\ref{sec:PACBayes} on other approximations that lead to several well-known PAC-Bayesian inequalities such as the ``PAC-Bayes-KL inequality'' \cite{seeger2002pac,maurer2004PACBayes}.

\begin{remark}[Related work]
	A variation of the IEI for the special case of the 0-1 loss appears in the monograph by Catoni \cite[Eq.1.2]{catonibook}, and has been rediscovered more recently for the sub-Gaussian loss in \cite{guedjInformationdensitySpectrumBassily2,guedjInformationdensitySpectrumBassily1}. The statements of \cite[Corollary 3, Eq. 20]{guedjInformationdensitySpectrumBassily2} and \cite[Corollary 6, Eq. 95]{guedjInformationdensitySpectrumBassily1} which are analogues of our Proposition \ref{thm:PACUnionboundMIsubGaussian} and Proposition \ref{thm:CMIPACBound01Loss}, Eq. \ref{eq:CMIPACBound01Loss}, resp., are incorrect as they assume that $\beta$ can be optimized ``for free,'' when in fact we have to pay a union bound price for optimizing $\beta$, which is selected \emph{before} the draw of the training sample. 
	We also note two related works that focus exclusively on unifying either PAC-Bayesian bounds for the 0-1 loss  \cite{PACMDLcommunicationcomplexity}, or information-theoretic bounds for the sub-Gaussian loss \cite{DziugaiteRoySGLDSteinkeNIPS2020}.
\end{remark}

\section{Differentially private data-dependent priors}\label{sec:DPPriors}
A PAC-Bayesian bound such as \eqref{eq:TongZhangHighProbBound} stipulates that the prior $\prior$ be chosen before the draw of the training sample $S$. $\prior$ may depend on the data generating distribution $\mu$ \cite{leverPACBayes}. However, our access to $\mu$ is only through $S$. To have a good control over the KL term in \eqref{eq:TongZhangHighProbBound}, it is desirable that $\prior$ be ``aligned'' with the data-dependent posterior $\post$. One way to achieve this goal is to choose $\prior$ based on $S$ in a differentially private fashion so that $\prior$ is stable to local perturbations in $S$ \cite{dziugaitediffprivateprior}. We can then treat $\prior$ ``as if'' it is independent of $S$. Here, the key quantity of interest is the approximate max-information between the input $S$ and the data-dependent prior. We shall make these notions precise. 

For $\alpha\ge 0$, the \emph{$\alpha$-approximate max-divergence} is defined as $$D_{\infty}^{\alpha}(P\|Q)= \ln \sup_{\cO \subseteq \cX:\,P(\cO)>\alpha} \frac{P(\cO)-\alpha}{Q(\cO)}.$$
The \emph{max-divergence} $D_{\infty}(P\|Q)$ is defined as $D_{\infty}^{\alpha}(P\|Q)$ for $\alpha = 0$. For a pair of variables $(X,Y)$ with joint law $P_{XY}$ and marginals $P_X$ and $P_Y$, the \emph{$\alpha$-approximate max-information} between $X$ and $Y$ is defined as $I_{\infty}^{\alpha}(X;Y) = D_{\infty}^{\alpha}(P_{XY}\|P_X\otimes P_Y)$.
The \emph{max-information} $I_{\infty}(X ; Y)$ is defined to be $I_{\infty}^{\alpha}(X ; Y)$ for $\alpha = 0$. $I_{\infty}(X;Y)$ is an upper bound on the ordinary mutual information $I(X;Y)$ \cite{feldman2015generalization}.
\begin{definition}[Differential Privacy \cite{generalizationRaginskyDworkDiffPrivacyBOOK}]
	For any $\epsilon > 0$ and $\delta\in [0,1]$, an algorithm $P_{W|S}$ is said to be \emph{$(\epsilon,\delta)$-differentially private} if for all pairs of datasets $s,s'\in\cZ^n$ that differ in a single element, 
	$D_{\infty}^{\delta}(P_{W|S=s}\|P_{W|S=s'})\le \epsilon$. The case $\delta=0$ is called pure differential privacy.
\end{definition}
\begin{definition}[{Max-Information of an algorithm \cite{feldman2015generalization}}]\label{def:MaxInfoOfAlgorithm}
	We say that an algorithm $P_{W|S}$ has $\alpha$-approximate max-information of $k$, denoted as $I_{\infty, \mu}^{\alpha}(P_{W|S}, n) \leq k$, if for every distribution $\mu$ over $\cZ$, we have $I_{\infty}^{\alpha}(S;W) \leq k$ when $S\sim \mu^{\otimes n}$. 
\end{definition}
It follows from the definition of $\alpha$-approximate max-information that if an algorithm $P_{W|S}$ has bounded approximate max-information, then we can control the probability of ``bad events'' that may arise as a result of the dependence of the output $W$ on the input $S$ \cite{feldman2015generalization}.
Let $S'\independent W$ be an independent sample with the same distribution as $S$. 
If for some $\alpha\ge 0$, $I_{\infty}^{\alpha}(S;W)= k$, then for any event $\cO\subseteq \cZ^n\times \cW$, we have
\begin{align}\label{eq:MaxInfoDependence}
\Pr((S,W)\in \cO) \le e^k \cdot \Pr((S',W)\in \cO) + \alpha.
\end{align}
Pure differential privacy implies a bound on the approximate max-information:
\begin{theorem}[{Pure differential privacy and $\alpha$-approximate max-information \cite[Theorem 20]{feldman2015generalization}}] \label{thm:max-infoDP}
	If $P_{W|S}$ is an $(\epsilon, 0)$-differentially private algorithm, then $I_{\infty, \mu}(P_{W|S}, n)\le n\epsilon$, and for any $\alpha > 0$, 
	$I_{\infty, \mu}^{\alpha}(P_{W|S}, n) \leq {n\epsilon^2}/{2} + \epsilon\sqrt{{n}\ln ({2}/{\alpha}) /2}$.
\end{theorem}
\begin{remark}\label{rem:max-infoDP}
	The result above is extended to $(\epsilon,\delta)$-differential privacy in \cite[Theorem 3.1]{rogers2016max}: If $P_{W|S}$ is an $(\epsilon,\delta)$-differentially
	private algorithm for $\epsilon \in (0,1/2]$ and $\delta \in (0,\epsilon)$, then for $\alpha =   O(n \sqrt{{\delta}/{\epsilon}})$, $I^\alpha_{\infty,\mu}(\cA, n) = O(n\epsilon^2  + n \sqrt{{\delta}/{\epsilon}})$. 
\end{remark}
\begin{restatable}{proposition}{PACDPPrior}\label{thm:PACDPPrior}
	Consider the setting in Theorem \ref{thm:TongZhangIRM}. Let $\prior^0 \in \cK(\cS,\cW)$ be an $(\epsilon, 0)$-differentially private algorithm. 
	Then with probability of at least $1-\delta$ over the choice of $S \sim\mu^{\otimes{n}}$, for all $\post\in\cM(\cW)$,
	\begin{align*}
	\E_{\post} [M_{\beta}&(W)] \leq \E_{\post}[L_S(W)]
	+ \frac{D(\post\|\prior^0(S))+f_n(\delta,\epsilon)}{n\beta}.
	\end{align*}
	where $f_n(\delta,\epsilon):=\ln\frac{2}{\delta} + \frac{n\epsilon^2}{2} + \epsilon\sqrt{\frac{n}{2}\ln \frac{4}{\delta}}$.
\end{restatable}
By Remark~\ref{rem:max-infoDP}, the result above can be extended to $(\epsilon,\delta)$-differentially private priors. Proposition~\ref{thm:PACDPPrior} is similar in spirit to the traditional PAC-Bayesian bounds in \cite[Theorem 4.2]{dziugaitediffprivateprior}, and \cite[Eq. 7]{rivasplatapac}, 
which, however, either apply only when the loss is bounded in $[0,1]$, or entails approximating a suitable exponential moment involving the true risk. 
We can also bound the expected generalization error. 
The next result follows from Theorems \ref{thm:max-infoDP} and {\ref{cor:TongZhangIRM}}:
\begin{corollary}\label{cor:PACDPPriorCMI}
	Consider the setting in {Theorem \ref{cor:TongZhangIRM}}. Let $\prior^0 \in {\cK(\cS,\cW)}$ be an $(\epsilon, 0)$-differentially private algorithm. Then with probability of at least $1-\delta$ over a draw of the sample $S$, for all $\post$, we have
	\begin{align*}
	\E_{\post} [\mathrm{g}(W,S)] &\le \frac{D(\post\|\prior^0(S))+f_n(\delta,\epsilon)}{n\beta}+{\frac{\psi(\beta)}{\beta}},
	\end{align*}
	where $f_n(\delta,\epsilon):=\ln\frac{2}{\delta} + \frac{n\epsilon^2}{2} + \epsilon\sqrt{\frac{n}{2}\ln \frac{4}{\delta}}$.
\end{corollary}
The main advantage of the max-information formulation is that we can get high probability guarantees at the cost of a $O(n\epsilon^2  + \epsilon\sqrt{n\ln{1/\delta}})$ correction term. This cost is compensated for by a lower KL complexity since the prior is more ``aligned'' with the data-dependent posterior than when chosen independently of the data. 
As is well-known \cite{feldman2015generalization,feldman2018calibrating}, a small mutual information between the data and the prior will not ensure that bad events will happen with low probability.

\section{Information complexity minimization}\label{sec:ICM}
Given any prior $\prior$, minimizing the right hand side of \eqref{eq:TongZhangHighProbBound} gives rise to the \emph{Information Complexity Minimization (ICM)} framework \cite{TongZhangPACBayes2006,TongZhangPACBayes2006AnnalsStats}. Concretely, for a given prior $\prior$ and hypothesis set $\cG \subseteq \cM(\cW)$, define the \emph{Optimal Information Complexity ($\mathrm{OIC}$)} at a given $\beta$ as 
\begin{align}\label{eq:InfoComplexity}
\mathrm{OIC}_{\cG}^{\beta}:=\inf_{\post\in\cG}\left\{\E_{\post}[L_S(W)] + {(n\beta)}^{-1}D(\post\|\prior)\right\}.
\end{align}
When $\cG=\cM(\cW)$, applying Lemma~\ref{thm:DVlemmalikeCatoni} to $f(w)=nL_S(w)$, and writing $\beta$ for $n\beta$, we obtain the Gibbs measure, $\post^{\star}$, in which case the $\mathrm{OIC}$ evaluates to the \emph{(extended) stochastic complexity} \cite{rissanen1998book,yamanishi1998} $$-{\beta}^{-1}\ln \E_{\prior} [e^{-\beta L_S(W)}].$$ The latter in turn coincides with the negative log-marginal likelihood for $\beta=1$ and the logarithmic loss function \cite{TongZhangPACBayes2006,TongZhangPACBayes2006AnnalsStats,Barroncover1991}. 

We briefly discuss two practical examples of ICM for learning with neural networks (NNs), namely, PAC-Bayes-SGD \cite{dziugaitenonvacuous} and Entropy-SGD \cite{FlatminimaPratik}, which can be viewed as optimization schemes that search for a ``flat'' minimum of the empirical loss surface \cite{FlatMinimaSchmidhuber}. We also show a PAC-Bayesian bound motivated by an Occam factor argument \cite{FlatMinimaBayesianMackay1992} in relation to flat minima.

\subsection{PAC-Bayes-SGD}\label{sec:PAC-Bayes-SGD}
\emph{PAC-Bayes-SGD} is an approach to computing generalization bounds for overparameterized NN classifiers trained with stochastic gradient descent (SGD) \cite{dziugaitenonvacuous}. These bounds are obtained by \emph{retraining} the network using an objective derived from a PAC-Bayes bound, starting from the solution found by SGD (or in fact any other procedure) for the training loss $L_S(w)$ w.r.t.\ $w$. 
The underlying hypothesis is that SGD finds ``good'' solutions that generalize well on unseen data only if such solutions are surrounded by a large volume of equally good solutions.
The method draws on an earlier work by Langford and Caruana \cite{langfordNotBounding}, and is closely related to the \emph{bits-back argument} due to Hinton and van Camp \cite{HintonBitsBack} (see Appendix \ref{sec:SFRL}). 
We show how Catoni's bound in Corollary~\ref{thm:pacbayesCatoniDiff} can be used to derive a PAC-Bayes-SGD objective. 

Consider a binary classification setting with examples domain $\cZ=\cX\times \{0,1\}$ and loss $\ell\colon\Rb^k\times\cZ\to\{0,1\}$. Each $w\in\cW$ corresponds to a classifier $f_w\colon\cX \to \{0,1\}$ that can be interpreted as a deterministic NN with parameters in $\Rb^k$. For trainable parameters $w_{\post}\in\Rb^k$, $\gamma\in\Rb^k_+$, $\lambda\in\Rb_+$, let $\cG$ be the set of all Gaussian posteriors of the form $\post = \cN(w_{\post},\diag(\gamma))$ and let $\prior=\cN(w_0,\lambda I_k)$ be a prior centered at a non-trainable random initialization, $w_0\in\Rb^k$. 
We can use a convex surrogate of the 0-1 loss, and the reparameterization trick $w = w_{\post}+ \nu \odot \sqrt{\gamma}$, $\nu\sim \cN(0,I_k)$ \cite{blundellBayesbyBackprop} to compute an unbiased estimate of the gradient of the PAC-Bayes bound in Corollary~\ref{thm:pacbayesCatoniDiff} w.r.t.\ the parameters $w_P,\gamma,\lambda$ and $\beta$. 
Computing the expectation $\E_{\post}[L_S(f_W)]$ is difficult in practice. Instead, we can use a Monte Carlo estimate $\hat{L}_S(f_W)=\tfrac{1}{m}\sum_{i=1}^{m} L_S(f_{W_i})$, where $W_i\overset{\text{i.i.d.}}{\sim}{\post}$. Then Corollary \ref{thm:pacbayesCatoniDiff} takes the form: For any $\delta,\delta'\in(0,1)$, fixed $\alpha>1$, $c\in (0,1)$, $b\in\mathbb{N}$, and $m,n\in\mathbb{N}$, with probability of at least $1-\delta-\delta'$ over a draw of $S\sim\mu^{\otimes n}$ and $W\sim (\post)^{\otimes m}$, 
\begin{align*}
\E_{\post}[L_{\mu}(f_W)] \le \inf_{\post\in\cG,\beta>1,\lambda\in (0,c)}&\Phi_{\beta}^{-1}\Big\{\hat{L}_S(f_W)+ \frac{\alpha}{n\beta}D(\post\|\prior)\notag+R(\lambda,\beta;\delta,\delta')\Big\},
\end{align*}
where $R=\frac{2\alpha}{n\beta}\ln \left[\frac{\ln \alpha^2\beta n}{\ln \alpha}\right]+\frac{\alpha}{n\beta}\ln\left[\tfrac{\pi^2 b^2}{6\delta}\left(\ln\tfrac{c}{\lambda}\right)^2\right]+\sqrt{\tfrac{1}{2m}\ln\tfrac{2}{\delta'}}$ accounts for the cost of optimizing the parameters $\beta,\,\lambda$, and using the Monte Carlo estimate of the empirical risk. For large $n,\,m$, $R$ is negligible, and the optimization is dominated by the IC term, $\hat{L}_S(f_W)+ \alpha {(n\beta)}^{-1} D(\post\|\prior)$.
\subsection{Entropy-SGD}
A related approach is \emph{Entropy-SGD} \cite{FlatminimaPratik}, which directly minimizes the stochastic complexity, $$-{\beta}^{-1}\ln \E_{\prior} e^{-\beta L_S(W)}.$$ This, however, entails optimizing the prior $\prior$, when ideally $\prior$ must be chosen before the draw of the training sample $S$. We can sample $Q$ instead in a differentially private fashion, and this forms the basis of the Entropy-SGLD algorithm \cite{dziugaite2017entropysgd}. For $\prior = \cN\big(w,(\beta \gamma)^{-1} I_k\big)$, the stochastic complexity can be equivalently written (up to constant terms) as $$-{\beta}^{-1} \ln \int_{w'\in\Rb^k}  e^{-\beta\big[ {L_S}\left(w^{\prime}\right)+ \tfrac{\gamma}{2}\|w-w'\|^2\big]} \mathrm{d} w^{\prime},$$ which can be interpreted as a measure of \emph{flatness} of the loss surface that measures the log-volume of low-loss parameter configurations around $w$. 
From the perspective of ICM, both Entropy- and PAC-Bayes- SGD can be viewed as optimization schemes that search for flat minima solutions. 

\subsection{PAC-Bayes and Occam factor}\label{sec:PACOccam}
Lemma \ref{lemma:GibbsLaplaceApprox} gives the form of the optimal posterior under a quadratic approximation of the loss around a local minimizer:%
\begin{restatable}{lemma}{GibbsLaplaceApprox}\label{lemma:GibbsLaplaceApprox}
	Consider a quadratic approximation of the training loss around a local minimizer $w_{\post}$, $\tilde{L}_S(w) = \tfrac{1}{2}(w - w_{\post})^\top H (w - w_{\post})$ where $H = \nabla^2 L_S(w)\vert_{w=w_\post}$, a fixed prior $\prior = \cN(w_{\prior},{\lambda}^{-1}I_k)$, and a posterior distribution of the form $\post = \cN(w_{\post},\Sigma_{\post})$. Then the solution to the convex optimization problem $\min_{\Sigma_{\post}} \E_{\post}[\tilde{L}_S(W)] +{(n\beta)}^{-1} D(\post\|\prior)$, is given by $\Sigma_{\post}^{\star} = H_{\lambda}^{-1}$, where $H_{\lambda}:=\left(n\beta H + \lambda I_k\right)$. Here we assume $\lambda>0$ is sufficiently large so that $H_\lambda$ is positive definite. 
\end{restatable}
We can use a posterior of the form $\post = \cN(w_{\post},H_{\lambda}^{-1})$ to get the following PAC-Bayesian bound that incorporates second-order curvature information of the training loss:
\begin{restatable}{proposition}{OccamPAC}\label{thm:OccamPAC}
	Let  $\{\lambda_i\}_{i=1}^k$ be the eigenvalues of $H_{\lambda}$ and suppose that $\lambda_i\ge \lambda>0$ for all $i$. Let $\prior = \cN(w_{\prior},{\lambda}^{-1}I_k)$ be a prior, and let $\post = \cN(w_{\post},H_{\lambda}^{-1})$. Then with probability of at least $1-\delta$ over a draw of the sample $S$, we have
	\begin{align}\label{eq:OccamPAC}
	\E_{\post} [M_{\beta}(W)] &\leq \E_{\post}[L_S(W)] + \frac{1}{n\beta}\ln \frac{1}{\delta}+ \frac{1}{n\beta}\left(\frac{\lambda}{2}\|w_{\prior}-w_{\post}\|^2+\frac{1}{2}\sum_{i=1}^{k}\ln \frac{\lambda_i}{\lambda}\right).
	\end{align}
\end{restatable}
Notably, the log-ratio term in \eqref{eq:OccamPAC}, $$\frac{1}{2}\sum_{i=1}^{k}\ln \frac{\lambda_i}{\lambda}=-\ln \sqrt{\det \tfrac{\lambda}{H_{\lambda}}}$$ is the negative logarithm of the \emph{Occam factor} \cite{FlatMinimaBayesianMackay1992,FlatminimaBayesianSGDMackay1992,energyentropycompetition}. The Occam factor can be interpreted as the fraction of the prior parameter space that is consistent with the training data. The log-Occam factor is the differential entropy associated with a Gaussian posterior with scaled covariance $\lambda (H_{\lambda})^{-1}$, and can be interpreted as the amount of information we gain about the model's parameters after seeing the training data. From the perspective of ICM, minimizing the right hand side of \eqref{eq:OccamPAC} w.r.t. the posterior leads to solutions with higher entropy and hence wider minima.

\section{Discussion}
We presented a unified treatment of PAC-Bayesian and information-theoretic generalization bounds starting from a fundamental information-theoretic inequality. Besides recovering several well-known bounds in the literature, we also obtained new bounds for data-dependent priors and unbounded loss functions. The bounds we studied are along the notion that bounded information (between the training data and the output hypothesis) implies learning. On the other hand, it is known that learning does \emph{not} imply bounded information \cite{generalizationRaginskyBassily,generalizationRaginskyBassily2}. In particular, the information revealed by a learning algorithm about its input can be unbounded even for hypothesis classes of VC dimension~1.  A result in a similar vein appears in the PAC-Bayesian framework \cite{feldmanAlonLimitationsPACBayes}. Identifying the common structural properties of these negative results in the information-theoretic and PAC-Bayesian frameworks is an important avenue for further investigation.

\section*{Acknowledgment}
This project has received funding from the European Research Council (ERC) under the EU's Horizon 2020 research and innovation programme (grant agreement n\textsuperscript{o} 757983).

\bibliographystyle{IEEEtran}
\bibliography{IEEEabrv,general}

\begin{thebibliography}{10}
\providecommand{\url}[1]{#1}
\csname url@samestyle\endcsname
\providecommand{\newblock}{\relax}
\providecommand{\bibinfo}[2]{#2}
\providecommand{\BIBentrySTDinterwordspacing}{\spaceskip=0pt\relax}
\providecommand{\BIBentryALTinterwordstretchfactor}{4}
\providecommand{\BIBentryALTinterwordspacing}{\spaceskip=\fontdimen2\font plus
\BIBentryALTinterwordstretchfactor\fontdimen3\font minus
  \fontdimen4\font\relax}
\providecommand{\BIBforeignlanguage}[2]{{%
\expandafter\ifx\csname l@#1\endcsname\relax
\typeout{** WARNING: IEEEtran.bst: No hyphenation pattern has been}%
\typeout{** loaded for the language `#1'. Using the pattern for}%
\typeout{** the default language instead.}%
\else
\language=\csname l@#1\endcsname
\fi
#2}}
\providecommand{\BIBdecl}{\relax}
\BIBdecl

\bibitem{russoZou2016controlling}
D.~Russo and J.~Zou, ``Controlling bias in adaptive data analysis using
  information theory,'' in \emph{Proceedings of the 19th International
  Conference on Artificial Intelligence and Statistics (AISTATS)}, 2016, pp.
  1232--1240.

\bibitem{generalizationRaginsky}
A.~Xu and M.~Raginsky, ``Information-theoretic analysis of generalization
  capability of learning algorithms,'' in \emph{Advances in Neural Information
  Processing Systems}, 2017, pp. 2524--2533.

\bibitem{jiao2017dependence}
J.~Jiao, Y.~Han, and T.~Weissman, ``Dependence measures bounding the
  exploration bias for general measurements,'' in \emph{Proceedings of the IEEE
  International Symposium on Information Theory (ISIT)}.\hskip 1em plus 0.5em
  minus 0.4em\relax IEEE, 2017, pp. 1475--1479.

\bibitem{generalizationRaginskyBassily}
R.~Bassily, S.~Moran, I.~Nachum, J.~Shafer, and A.~Yehudayoff, ``Learners that
  use little information,'' in \emph{International Conference on Algorithmic
  Learning Theory (ALT)}, 2018, pp. 25--55.

\bibitem{veeravalli}
Y.~Bu, S.~Zou, and V.~V. Veeravalli, ``Tightening mutual information based
  bounds on generalization error,'' in \emph{Proceedings of the IEEE
  International Symposium on Information Theory (ISIT)}.\hskip 1em plus 0.5em
  minus 0.4em\relax IEEE, 2019, pp. 587--591.

\bibitem{issa2019strengthened}
I.~Issa, A.~R. Esposito, and M.~Gastpar, ``Strengthened information-theoretic
  bounds on the generalization error,'' in \emph{Proceedings of the IEEE
  International Symposium on Information Theory (ISIT)}.\hskip 1em plus 0.5em
  minus 0.4em\relax IEEE, 2019, pp. 582--586.

\bibitem{steinke2020reasoning}
T.~Steinke and L.~Zakynthinou, ``Reasoning about generalization via conditional
  mutual information,'' in \emph{Conference On Learning Theory}, 2020, pp.
  3437--3452.

\bibitem{steinke2020reasoningOpen}
------, ``Open problem: {I}nformation complexity of {VC} learning,'' in
  \emph{Conference on Learning Theory}, 2020, pp. 3857--3863.

\bibitem{DziugaiteRoySGLDSteinke}
M.~Haghifam, J.~Negrea, A.~Khisti, D.~M. Roy, and G.~K. Dziugaite, ``Sharpened
  generalization bounds based on conditional mutual information and an
  application to noisy, iterative algorithms,'' \emph{Advances in Neural
  Information Processing Systems}, vol.~33, pp. 9925--9935, 2020.

\bibitem{bousquet2002stability}
O.~Bousquet and A.~Elisseeff, ``Stability and generalization,'' \emph{Journal
  of Machine Learning Research}, vol.~2, no. Mar, pp. 499--526, 2002.

\bibitem{ShalevLearnabilityStabilityUnifConvergence}
S.~Shalev-Shwartz, O.~Shamir, N.~Srebro, and K.~Sridharan, ``Learnability,
  stability and uniform convergence,'' \emph{The Journal of Machine Learning
  Research}, vol.~11, pp. 2635--2670, 2010.

\bibitem{bassily2016algorithmic}
R.~Bassily, K.~Nissim, A.~Smith, T.~Steinke, U.~Stemmer, and J.~Ullman,
  ``Algorithmic stability for adaptive data analysis,'' in \emph{Proceedings of
  the 48th Annual ACM Symposium on Theory of Computing (STOC)}, 2016, pp.
  1046--1059.

\bibitem{feldman2015generalization}
C.~Dwork, V.~Feldman, M.~Hardt, T.~Pitassi, O.~Reingold, and A.~Roth,
  ``Generalization in adaptive data analysis and holdout reuse,'' in
  \emph{Advances in Neural Information Processing Systems}, 2015, pp.
  2350--2358.

\bibitem{rogers2016max}
R.~Rogers, A.~Roth, A.~Smith, and O.~Thakkar, ``Max-information, differential
  privacy, and post-selection hypothesis testing,'' in \emph{57th Annual
  Symposium on Foundations of Computer Science (FOCS)}.\hskip 1em plus 0.5em
  minus 0.4em\relax IEEE, 2016, pp. 487--494.

\bibitem{feldman2018calibrating}
V.~Feldman and T.~Steinke, ``Calibrating noise to variance in adaptive data
  analysis,'' in \emph{Conference On Learning Theory}, 2018, pp. 535--544.

\bibitem{generalizationRaginskyPACBayesMcallester}
D.~A. McAllester, ``{PAC}-{B}ayesian model averaging,'' in \emph{Proceedings of
  the 12th Annual Conference on Computational Learning Theory}.\hskip 1em plus
  0.5em minus 0.4em\relax ACM, 1999, pp. 164--170.

\bibitem{mcallester1999some}
------, ``Some {PAC}-{B}ayesian theorems,'' \emph{Machine Learning}, vol.~37,
  no.~3, pp. 355--363, 1999.

\bibitem{mcallesterdropout}
------, ``A {PAC}-{B}ayesian tutorial with a dropout bound,'' \emph{arXiv
  preprint arXiv:1307.2118}, 2013.

\bibitem{shalev2014understanding}
S.~Shalev-Shwartz and S.~Ben-David, \emph{Understanding machine learning:
  {F}rom theory to algorithms}.\hskip 1em plus 0.5em minus 0.4em\relax
  Cambridge University Press, 2014.

\bibitem{TongZhangPACBayes2006}
T.~Zhang, ``Information-theoretic upper and lower bounds for statistical
  estimation,'' \emph{IEEE Transactions on Information Theory}, vol.~52, no.~4,
  pp. 1307--1321, 2006.

\bibitem{catonibook}
O.~Catoni, \emph{{PAC}-{B}ayesian Supervised Classification: {T}he
  Thermodynamics of Statistical Learning}.\hskip 1em plus 0.5em minus
  0.4em\relax Institute of Mathematical Statistics, 2007, vol.~56.

\bibitem{maurer2004PACBayes}
A.~Maurer, ``A note on the {PAC} {B}ayesian theorem,'' \emph{arXiv preprint
  cs/0411099}, 2004.

\bibitem{TimvanErvenPACBayesTongZhang}
T.~van Erven, ``{PAC-B}ayes mini-tutorial: {A} continuous union bound,''
  \emph{arXiv preprint arXiv:1405.1580}, 2014.

\bibitem{grunwaldTongZhangFastrates}
P.~D. Gr{\"u}nwald and N.~A. Mehta, ``Fast rates for general unbounded loss
  functions: {F}rom {ERM} to generalized {B}ayes.'' \emph{Journal of Machine
  Learning Research}, vol.~21, no.~56, pp. 1--80, 2020.

\bibitem{alquierPACBayesPropertiesGibbs}
P.~Alquier, J.~Ridgway, and N.~Chopin, ``On the properties of variational
  approximations of {G}ibbs posteriors,'' \emph{The Journal of Machine Learning
  Research}, vol.~17, no.~1, pp. 8374--8414, 2016.

\bibitem{germain2009Paclinear}
P.~Germain, A.~Lacasse, F.~Laviolette, and M.~Marchand, ``{PAC-B}ayesian
  learning of linear classifiers,'' in \emph{Proceedings of the 26th
  International Conference on Machine Learning (ICML)}, 2009, pp. 353--360.

\bibitem{FlatminimaBayesianSGDGermainMackay1992}
P.~Germain, F.~Bach, A.~Lacoste, and S.~Lacoste-Julien, ``{PAC-B}ayesian theory
  meets {B}ayesian inference,'' in \emph{Advances in Neural Information
  Processing Systems}, 2016, pp. 1884--1892.

\bibitem{rivasplataThiemannPACBackprop}
N.~Thiemann, C.~Igel, O.~Wintenberger, and Y.~Seldin, ``A strongly quasiconvex
  {PAC-B}ayesian bound,'' in \emph{International Conference on Algorithmic
  Learning Theory (ALT)}, 2017, pp. 466--492.

\bibitem{rivasplatapac}
O.~Rivasplata, I.~Kuzborskij, C.~Szepesv{\'a}ri, and J.~Shawe-Taylor,
  ``{PAC-B}ayes analysis beyond the usual bounds,'' in \emph{Advances in Neural
  Information Processing Systems}, vol.~33, 2020.

\bibitem{rivasplatapackuzborskij}
I.~Kuzborskij, N.~Cesa-Bianchi, and C.~Szepesv{\'a}ri, ``Distribution-dependent
  analysis of {G}ibbs-{ERM} principle,'' in \emph{Conference on Learning
  Theory}, 2019, pp. 2028--2054.

\bibitem{dziugaitenonvacuous}
G.~K. Dziugaite and D.~M. Roy, ``Computing nonvacuous generalization bounds for
  deep (stochastic) neural networks with many more parameters than training
  data,'' in \emph{Proceedings of the 33rd Conference on Uncertainty in
  Artificial Intelligence (UAI)}, 2017.

\bibitem{FlatMinimaSchmidhuber}
S.~Hochreiter and J.~Schmidhuber, ``Flat minima,'' \emph{Neural Computation},
  vol.~9, no.~1, pp. 1--42, 1997.

\bibitem{boucheron2013concentration}
S.~Boucheron, G.~Lugosi, and P.~Massart, \emph{Concentration inequalities: {A}
  nonasymptotic theory of independence}.\hskip 1em plus 0.5em minus 0.4em\relax
  Oxford University Press, 2013.

\bibitem{TongZhangPACBayes2006AnnalsStats}
T.~Zhang, ``From $\varepsilon$-entropy to {KL}-entropy: {A}nalysis of minimum
  information complexity density estimation,'' \emph{The Annals of Statistics},
  vol.~34, no.~5, pp. 2180--2210, 2006.

\bibitem{seeger2002pac}
M.~Seeger, ``{PAC-B}ayesian generalisation error bounds for {G}aussian process
  classification,'' \emph{Journal of Machine Learning Research}, vol.~3, no.
  Oct, pp. 233--269, 2002.

\bibitem{guedjInformationdensitySpectrumBassily2}
F.~Hellstr{\"o}m and G.~Durisi, ``Generalization error bounds via $m$th central
  moments of the information density,'' in \emph{Proceedings of the IEEE
  International Symposium on Information Theory (ISIT)}.\hskip 1em plus 0.5em
  minus 0.4em\relax IEEE, 2020, pp. 2741--2746.

\bibitem{guedjInformationdensitySpectrumBassily1}
------, ``Generalization bounds via information density and conditional
  information density,'' \emph{IEEE Journal on Selected Areas in Information
  Theory}, pp. 824--839, 2020.

\bibitem{PACMDLcommunicationcomplexity}
A.~Blum and J.~Langford, ``{PAC-MDL} bounds,'' in \emph{Learning theory and
  kernel machines}.\hskip 1em plus 0.5em minus 0.4em\relax Springer, 2003, pp.
  344--357.

\bibitem{DziugaiteRoySGLDSteinkeNIPS2020}
H.~Hafez-Kolahi, Z.~Golgooni, S.~Kasaei, and M.~Soleymani, ``Conditioning and
  processing: {T}echniques to improve information-theoretic generalization
  bounds,'' in \emph{Advances in Neural Information Processing Systems},
  vol.~33, 2020.

\bibitem{leverPACBayes}
G.~Lever, F.~Laviolette, and J.~Shawe-Taylor, ``Tighter {PAC-B}ayes bounds
  through distribution-dependent priors,'' \emph{Theoretical Computer Science},
  vol. 473, pp. 4--28, 2013.

\bibitem{dziugaitediffprivateprior}
G.~K. Dziugaite and D.~M. Roy, ``Data-dependent {PAC-B}ayes priors via
  differential privacy,'' in \emph{Advances in Neural Information Processing
  Systems}, 2018, pp. 8430--8441.

\bibitem{generalizationRaginskyDworkDiffPrivacyBOOK}
C.~Dwork and A.~Roth, ``The algorithmic foundations of differential privacy,''
  \emph{Foundations and Trends{\textregistered} in Theoretical Computer
  Science}, vol.~9, no. 3--4, pp. 211--407, 2014.

\bibitem{rissanen1998book}
J.~Rissanen, \emph{Stochastic complexity in statistical inquiry}.\hskip 1em
  plus 0.5em minus 0.4em\relax World scientific, 1989.

\bibitem{yamanishi1998}
K.~Yamanishi, ``A decision-theoretic extension of stochastic complexity and its
  applications to learning,'' \emph{IEEE Transactions on Information Theory},
  vol.~44, no.~4, pp. 1424--1439, 1998.

\bibitem{Barroncover1991}
A.~R. Barron and T.~M. Cover, ``Minimum complexity density estimation,''
  \emph{IEEE Transactions on Information Theory}, vol.~37, no.~4, pp.
  1034--1054, 1991.

\bibitem{FlatminimaPratik}
P.~Chaudhari, A.~Choromanska, S.~Soatto, Y.~LeCun, C.~Baldassi, C.~Borgs,
  J.~Chayes, L.~Sagun, and R.~Zecchina, ``Entropy-{SGD}: {B}iasing gradient
  descent into wide valleys,'' in \emph{International Conference on Learning
  Representations}, 2017.

\bibitem{FlatMinimaBayesianMackay1992}
D.~J.~C. MacKay, ``A practical {B}ayesian framework for backpropagation
  networks,'' \emph{Neural computation}, vol.~4, no.~3, pp. 448--472, 1992.

\bibitem{langfordNotBounding}
J.~Langford and R.~Caruana, ``({N}ot) bounding the true error,'' in
  \emph{Advances in Neural Information Processing Systems}, 2002, pp. 809--816.

\bibitem{HintonBitsBack}
G.~E. Hinton and D.~van Camp, ``Keeping neural networks simple by minimising
  the description length of weights,'' in \emph{Conference On Learning Theory},
  1993, pp. 5--13.

\bibitem{blundellBayesbyBackprop}
C.~Blundell, J.~Cornebise, K.~Kavukcuoglu, and D.~Wierstra, ``Weight
  uncertainty in neural networks,'' in \emph{Proceedings of the 32nd
  International Conference on Machine Learning (ICML)}, 2015, pp. 1613--1622.

\bibitem{dziugaite2017entropysgd}
G.~K. Dziugaite and D.~M. Roy, ``Entropy-{SGD} optimizes the prior of a
  {PAC-B}ayes bound: {G}eneralization properties of {E}ntropy-{SGD} and
  data-dependent priors,'' in \emph{Proceedings of the 35th International
  Conference on Machine Learning (ICML)}, 2018, pp. 1377--1386.

\bibitem{FlatminimaBayesianSGDMackay1992}
S.~L. Smith and Q.~V. Le, ``A {B}ayesian perspective on generalization and
  stochastic gradient descent,'' in \emph{International Conference on Learning
  Representations}, 2018.

\bibitem{energyentropycompetition}
Y.~Zhang, A.~M. Saxe, M.~S. Advani, and A.~A. Lee, ``Energy--entropy
  competition and the effectiveness of stochastic gradient descent in machine
  learning,'' \emph{Molecular Physics}, vol. 116, no. 21-22, pp. 3214--3223,
  2018.

\bibitem{generalizationRaginskyBassily2}
I.~Nachum and A.~Yehudayoff, ``Average-case information complexity of
  learning,'' in \emph{International Conference on Algorithmic Learning Theory
  (ALT)}, 2019, pp. 633--646.

\bibitem{feldmanAlonLimitationsPACBayes}
R.~Livni and S.~Moran, ``A limitation of the {PAC-B}ayes framework,'' in
  \emph{Advances in Neural Information Processing Systems}, vol.~33, 2020.

\bibitem{harshacommunicationcomplexity}
P.~Harsha, R.~Jain, D.~McAllester, and J.~Radhakrishnan, ``The communication
  complexity of correlation,'' \emph{IEEE Transactions on Information Theory},
  vol.~56, no.~1, pp. 438--449, 2009.

\bibitem{havasibitsback}
M.~Havasi, R.~Peharz, and J.~M. Hern{\`a}ndez-Lobato, ``Minimal random code
  learning: {G}etting bits back from compressed model parameters,'' in
  \emph{International Conference on Learning Representations}, 2019.

\bibitem{CTLiSFRL}
C.~T. Li and A.~El~Gamal, ``Strong functional representation lemma and
  applications to coding theorems,'' \emph{IEEE Transactions on Information
  Theory}, vol.~64, no.~11, pp. 6967--6978, 2018.

\end{thebibliography}

\newpage
\appendix
\section{The strong functional representation lemma and single-draw\\ bounds} \label{sec:SFRL}
In this section, we highlight a functional characterization of the mutual information in relation to a single-draw generalization bound of the form,
$\Pr(\vert \mathrm{g}(W,S)\vert > \epsilon)\le \delta$, due to \cite{generalizationRaginskyBassily}.

A randomized learning algorithm $P_{W|S}$ can be viewed as a noisy channel that maps the input sample $S$ to conditional distributions of hypotheses $W$ in $\cW$. Consider the one-shot noisy channel simulation problem \cite{harshacommunicationcomplexity}. Alice and Bob share a common random string $R$, possibly of unbounded length, generated in advance. 
Alice observes a sample $s\in\cZ^n$ drawn according to $P_S$ and communicates a prefix-free message $M$ to Bob via a noiseless channel such that Bob can output a hypothesis $w\in\cW$ that is distributed according to $P_{W|S=s}$. Harsha \textit{et al}. \cite{harshacommunicationcomplexity} showed that the minimum expected description length of $M$ (in bits) needed to accomplish this task is roughly equal to the input-output mutual information $I(S;W)$.
Variations on this theme have appeared in a learning-theoretic setting \cite{PACMDLcommunicationcomplexity}, and by way of the \emph{bits-back argument} due to \cite{HintonBitsBack};
see, e.g., \cite{havasibitsback}. 
More generally, we note the following functional characterization of the mutual information:
\begin{theorem}[Strong functional representation lemma (SFRL)  \cite{CTLiSFRL}]\label{thm:SFRL}
	For any pair of jointly distributed random variables $(S,W)$ 
	with $I(S;W) < \infty$, there exists a random variable $R$ independent of $S$ such that $W$ can be represented as a deterministic function of $S$ and $R$, and 
	\vspace{.24cm}
	\begin{align*}
	I(S;W) \le H(W |R) \le I(S;W) + \log(I(S;W) + 1) + 4.
	\end{align*}
\end{theorem}
The SFRL implies the existence of a random variable $R\independent S$ such that $H(W|R) \approx I(S;W)$. In the one-shot channel simulation problem, for instance, $R$ encapsulates the common randomness shared between Alice and Bob.

Consider the case for the $\{0,1\}$-valued loss. 
If the algorithm $P_{W|S}$ is deterministic, then we have $I:=I(S;W)=H(W)$.
By Markov's inequality, with probability of at least $1-\delta$, we have $P_W(w) \ge e^{-I / \delta}$. Let $\cW_0\subseteq \cW$ be the set of hypotheses so that $P_{W}(w) \geq e^{-I/\delta}$. The size of $\cW_0$ is at most $e^{I/\delta}$ since $1=\Pr(\cW)\ge \Pr(\cW_0)=\textstyle\sum_{w\in\cW_0}P_W(w) \ge |\cW_0| e^{-I/\delta}$. By the Chernoff-Hoeffding bound, for every $w$ in $\cW$, 
$$\Pr_S\Big(\vert \mathrm{g}(w,S)\vert > \epsilon\Big) \le  2e^{-2n\epsilon^2}\quad \forall\epsilon>0.$$
Applying the union bound over all $w\in\cW_0$, the probability of error for the algorithm is $2|\cW_0|e^{-2n\epsilon^2}+\delta$, where the second summand is for the case where the algorithm outputs a function outside $\cW_0$. Hence, for every $w \in \cW_0$, 
the empirical risk is close to the true risk 
for $n = \Omega(\tfrac{I}{\delta\epsilon^2})$ with probability of at least $1-\delta$.	

Any randomized algorithm can be simulated by randomly sampling a deterministic algorithm from some distribution $R$ before observing the input $S$. By the SFRL, these algorithms have the property that on average (over $R$), $H(W|R) \approx I(S;W)$. Using the argument for the deterministic case and integrating over $R$,
we can bound the probability of error for the randomized case as
\begin{align}
\Pr_{S,W}\Big(\vert \mathrm{g}(W,S)\vert > \epsilon\Big) = O\left(\frac{I(S;W)}{n\epsilon^2}\right).
\end{align}
An analogous bound for the sub-Gaussian loss appears in \cite[Theorem 3]{generalizationRaginsky}.

\section{Proofs}\label{sec:proofs}
\changelocaltocdepth{1}
\subsection{Proofs for Section \ref{sec:IEI}}
The following variational characterization of the KL divergence is a rephrasing of Lemma~\ref{thm:DVlemmalikeCatoni}:
\begin{lemma}[Donsker-Varadhan] \label{lem:DVlemma}
	Let $P,\,Q$ be probability measures on $\cW$, and let $\cF$ denote the set of real-valued measurable functions $f$ on $\cW$ such that $\E_Q[e^{f(W)}]< \infty$. 
	If $D(P\|Q)<\infty$, then for every $f\in\cF$, we have
	\begin{align*}
	D(P\|Q) = \sup_{f \in \cF} \Big\{ \E_P[f(W)]  -\ln \E_Q[e^{f(W)}] \Big\},
	\end{align*}
	where the supremum is attained when $f=\ln \frac{\mathrm{d}P}{\mathrm{d}Q}$.	
\end{lemma}

We include a proof of the information exponential inequality, 
since we will use the arguments. 
\begin{proof}[Proof of Lemma \ref{lem:InfoExpInequality}]
	Applying the Donsker-Varadhan Lemma \ref{lem:DVlemma} to the function, 
	\begin{align}\label{eq:TongZhangDV}
	f(w) = n\beta (M_{\beta}(w)-L_S(w)),
	\end{align}
	we obtain,
	\begin{align}\label{eq:TongZhangtemp}
	n\beta \E_{\post}[M_{\beta}(W)-L_S(W)] - D(\post\|\prior) \le \ln \E_{\prior}[e^{n\beta(M_{\beta}(W)-L_S(W))}].
	\end{align}	
	Exponentiating both sides of \eqref{eq:TongZhangtemp} and taking expectations w.r.t. $S\sim\mu^{\otimes n}$, we have
	\begin{align}\label{eq:TongZhangtemp1}
	{\E}_{S}\exp \big\{n\beta{\E}_{\post}\left[M_{\beta}(W)-L_S(W)\right]-D(\post \| \prior)\big\}\le {\E}_{S}\E_{\prior}[e^{n\beta(M_{\beta}(W)-L_S(W))}].
	\end{align}	
	Since $Z_i\overset{\text{i.i.d.}}{\sim}{\mu}$, for any $w\in\mathcal{W}$ and $\beta>0$, we have $e^{-n\beta M_{\beta}(w)}={\E}_{S\sim\mu^{\otimes n}}\left[e^{-n\beta L_S(w) }\right]$. 
	This observation and Fubini's theorem implies that the right hand side of \eqref{eq:TongZhangtemp1} is equal to one. This proves the IEI.
\end{proof}

\begin{proof}[Proof of Theorem~\ref{cor:TongZhangIRM}]
	Letting $R(S):=n\beta \E_{\post}[M_{\beta}(W)-L_S(W)] - D(\post\|\prior)$, by Lemma~\ref{lem:InfoExpInequality}, we have $\E_S[e^{R(S)}] \le 1$.
	By Markov's inequality, 
	\[\underset{S}{\Pr}\left(R(S) > \ln \frac{1}{\delta}\right) = \underset{S}{\Pr}\left(e^{R(S)} > \frac{1}{\delta}\right) \le \E_S[e^{R(S)}]\delta\le \delta.\] 
	Therefore, with probability of at least $1 -\delta$ over the choice of $S\sim\mu^{\otimes n}$, we have for all $\post\ll\prior$,
	\begin{align*}
	\E_{\post} [M_{\beta}(W)] \leq \E_{\post}[L_S(W)] + \frac{1}{n\beta}\left(D(\post\|\prior)+\ln \frac{1}{\delta}\right).
	\end{align*}
	By assumption \eqref{eq:AssumptionStar},
	\begin{align*}
	\sup_{w \in \cW} \left[L_{\mu}(w)-M_{\beta}(w)\right]\leq \frac{\psi(\beta)}{\beta},\, \forall\beta>0, 
	\end{align*}
	when \eqref{eq:TongZhangHighProbBoundOpt} follows.
	
	For the bound in expectation, note that by Jensen's inequality,
	$e^{\E_S[R(S)]}\le \E_S[e^{R(S)}]\le 1$, which implies $\E_S[R(S)]\le 0$, when we have
	\begin{align*}
	\E_{SW} [M_{\beta}(W)] \leq \E_{SW}[L_S(W)] + \frac{1}{n\beta }D(\post\|\prior|P_S).
	\end{align*}
	Using \eqref{eq:AssumptionStar}, and rearranging and optimizing, we have
	\begin{align*}
	\E_{SW} [\mathrm{g}(W,S)] &\le \inf_{\beta>0}\frac{\tfrac{1}{n}D(\post\|\prior|P_S) + \psi(\beta)}{\beta} \\&=\psi^{*-1}\left(\frac{D(\post\|\prior|P_S)}{n}\right),
	\end{align*}
	where the second equality follows from Lemma~\ref{lemma:psistar}.
\end{proof}

The proof of Proposition~\ref{thm:PACUnionboundMIsubGaussian} follows that of \cite[Lemma 8]{TimvanErvenPACBayesTongZhang}, extending it to sub-Gaussian losses.
\begin{proof}[Proof of Proposition~\ref{thm:PACUnionboundMIsubGaussian}] 
	For $0<u<v$, and $i=0,\ldots,\ceil{\log_{\alpha}\tfrac{v}{u}}-1$, for all $i$ let $\beta_i=u\alpha^i$ be selected before the draw of the training sample. Then for every $\beta\in[u,v]$, there is a $\beta_i$ such that $\beta_i\le \beta \le \alpha\beta_i$. 
	
	We can extend \eqref{eq:TongZhangHighProbBound} by applying a union bound over the $\beta_i$'s, so that for all $\post$ with probability of at least $1-\delta$ over the draw of $S$, the following holds simultaneously for all $\beta_i$:
	\begin{align}
	\E_{\post} [M_{\beta_i}(W)] \leq \E_{\post} [L_{S}(W)] + \frac{\alpha}{n\beta_i}\left(D(\post\|\prior)+ \ln{\frac{\ceil{\log_{\alpha}\tfrac{v}{u}}}{\delta}}\right).
	\end{align}
	By Proposition \ref{Prop:logmgfProperties} item \emph{2)}, for any $w\in\cW$, $M_{\beta}(w)$ is a nonincreasing of $\beta$. Thus for any $\beta\in[u,v]$ and $\beta_i$ such that $\beta_i\le \beta \le \alpha\beta_i$, $M_{\beta}(w)\le M_{\beta_i}(w)$ and $\tfrac{1}{\beta_i}\le \tfrac{\alpha}{\beta}$. Moreover, since $\ell(w,Z)$ is $\sigma$-sub-Gaussian under $\mu$ by assumption, we have for all $w\in\cW$ and $\beta>0$, $L_{\mu}(w) \le M_{\beta}(w) + \frac{\beta}{2}\sigma^2$. Hence, with probability of at least $1-\delta$ we have,
	\begin{align}\label{eq:PACBayesSubgammatemp}
	\E_{\post} [L_{\mu}(W)] \leq \E_{\post} [L_{S}(W)] + \frac{\alpha}{n\beta}\left(D(\post\|\prior)+ \ln{\frac{\ceil{\log_{\alpha}\tfrac{v}{u}}}{\delta}}\right) + \frac{\beta\sigma^2}{2}.
	\end{align}
	Letting $J = D(\post\|\prior)+\ln \frac{\log_{\alpha}\sqrt{n}+K}{\delta}$, we find that the value for $\beta$ that optimizes 
	the right hand side of the bound in the statement of the proposition
	is bounded from below by $\sqrt{\frac{2\alpha}{n\sigma^2}}$.
	Letting $u=\frac{1}{\sqrt{n}}\min\left\{\sqrt{\frac{2\alpha}{\sigma^2}},v\right\}$ and plugging it in \eqref{eq:PACBayesSubgammatemp} completes the proof.
\end{proof}

\subsection{Proofs for Section \ref{sec:CMIb}}
\begin{proof}[Proof of Proposition~\ref{thm:CMIPACBound01Loss}]
	Applying the Donsker-Varadhan lemma \ref{lem:DVlemma} to the function, 
	$f(w) = n\beta \mathrm{g}(w,\sZ,U)$, 
	and following the same steps as in the proof of Lemma \ref{lem:InfoExpInequality}, we arrive at
	\begin{align}\label{eq:CMItemp1}
	{\E}_{\sZ,U}\exp \big\{n\beta{\E}_{\post}[\mathrm{g}(W,\sZ,U)]-D(\post \| \prior)\big\}\le \E_{\prior}{\E}_{\sZ,U}[e^{n\beta\mathrm{g}(W,\sZ,U)}]=\E_{\prior}{\E}_{\sZ}\E_U[e^{n\beta\mathrm{g}(W,\sZ,U)}],
	\end{align}	
	where the last equality follows since $\sZ\independent U$. Since $\ell\in[0,1]$, $\mathrm{g}(W,\sZ,U)$ is $\tfrac{1}{\sqrt{n}}$-sub-Gaussian. 
	Morever, $\E_U[\mathrm{g}(W,\sZ,U)]=0$. 
	By Hoeffding’s lemma, we have ${\E}_{\sZ}\E_U[e^{n\beta\mathrm{g}(W,\sZ,U)}]\le e^{n\beta^2/2}$, and hence
	\begin{align}\label{eq:CMItemp2}
	{\E}_{\sZ,U}\exp \big\{n\beta{\E}_{\post}[\mathrm{g}(W,\sZ,U)]-D(\post \| \prior)-\tfrac{n\beta^2}{2}\big\}\le  1.
	\end{align}
	\eqref{eq:CMIPACBound01Loss} then follows by an application of Markov's inequality.
	
	Let $R(\sZ,U)=n\beta{\E}_{\post}[\mathrm{g}(W,\sZ,U)]-D(\post \| \prior)-\tfrac{n\beta^2}{2}$. From \eqref{eq:CMItemp2}, and using Jensen's inequality, we have
	$e^{\E_{\sZ,U}[R(\sZ,U)]}\le \E_{\sZ,U}[e^{R(\sZ,U)}]\le 1$, which implies 
	\[
	\E_{\sZ,U,W}[\mathrm{g}(W,\sZ,U)]\le \inf_{\beta>0} \left(\frac{D(\post_{W|\sZ U}\|\prior_{W|\sZ}|P_{\sZ U})}{n\beta}+\frac{\beta}{2}\right)=\sqrt{\frac{2\cdot D(\post_{W|\sZ U}\|\prior_{W|\sZ}|P_{\sZ U})}{n}},
	\] 
	and we have shown \eqref{eq:CMIPACBound01LossExpectation}.
\end{proof}
Under the oracle prior $\prior_{W|\sZ}=P_{W|\sZ}$, we have $D(\post_{W|\sZ U}\|P_{W|\sZ}|P_{\sZ U})=I(W;U|\sZ)$. By noting that $\E_{\sZ,U,W}[\mathrm{g}(W,\sZ,U)]=\E_{\sZ,U,W}[\mathrm{g}(W,\sZ_U)]$, we recover \cite[Theorem 2(1)]{steinke2020reasoning}.

\subsection{Omitted details in Section \ref{sec:CatoniMcAllRecovery}}\label{sec:PACBayes}
We show how inequality \eqref{eq:TongZhangHighProbBound} relates to other well-known PAC-Bayesian inequalities such as the ``PAC-Bayes-KL-inequality'' \cite{seeger2002pac,maurer2004PACBayes}. 
Applying the Donsker-Varadhan lemma to the function $f(w)=n\beta(L_{\mu}(w)-L_S(w))$, which involves the true risk $L_{\mu}(w)$ instead of the annealed expectation $M_{\beta}(w)$ (see \ref{eq:TongZhangDV}), and following the same steps as in the proof of \eqref{eq:TongZhangHighProbBound} in Theorem \ref{thm:TongZhangIRM},
we arrive at 
the following PAC-Bayesian bound:
\begin{align}\label{eq:TongZhangtempAlquir}
\Pr_{S\sim\mu^{\otimes{n}}}\Bigg(
\E_{\post} [L_{\mu}(W)] \leq 
\E_{\post} [L_{S}(W)] + \frac{1}{n\beta} \Big[&D(\post\|\prior)+ \ln{\frac{1}{\delta}}\notag\\&+\ln \E_{\prior} \E_{S'\sim \mu^{\otimes n}} \,e^{n\beta\big(L_{\mu}(W)-L_{S'}(W)\big)}\Big]
\Bigg)\ge 1-\delta. 
\end{align}
For an explicit comparison of \eqref{eq:TongZhangtempAlquir} with \eqref{eq:TongZhangHighProbBound}, we write the latter as
\begin{align}\label{eq:TongZhangtemp2}
\Pr_{S\sim\mu^{\otimes{n}}}\Bigg(
\E_{\post} [M_{\beta}(W)] \leq 
\E_{\post} [L_{S}(W)] + \frac{1}{n\beta} \Big[&D(\post\|\prior)+ \ln{\frac{1}{\delta}}\notag\\&+\underbrace{\ln \E_{\prior} \E_{S'\sim \mu^{\otimes n}} \,e^{n\beta\big(M_{\beta}(W)-L_{S'}(W)\big)}}_{=0}\Big]
\Bigg)\ge 1-\delta,
\end{align}
where the last term in the right hand side of the bound in \eqref{eq:TongZhangtemp2} vanishes since 
\[e^{-n\beta M_{\beta}(w)}={\E}_{S'\sim\mu^{\otimes n}}\left[e^{-n\beta L_{S'}(w) }\right]\]
for any $w\in\mathcal{W}$ and $\beta>0$. 
In contrast, the term $\ln \E_{\prior} \E_{S'\sim \mu^{\otimes n}} \,e^{n\beta(L_{\mu}(W)-L_{S'}(W))}$
involving the true risk in \eqref{eq:TongZhangtempAlquir} is, in general, positive. 

Specializing to the case of a $\{0,1\}$-valued loss, fix $\beta=1$, and let $\Delta:[0,1]\times [0,1]\to \Rb$ be a convex function. Applying the Donsker-Varadhan lemma to the function, $f(w)=n\Delta\left(L_{S}(w),L_{\mu}(w)\right)$, 
following the same steps as in the proof of \eqref{eq:TongZhangHighProbBound} in Theorem \ref{thm:TongZhangIRM}, 
and by noting that $$\Delta\left(\E_{\post}[L_S(W)],\E_{\post} [L_{\mu}(W)])\right)\leq \mathbb{E}_{\post}\left[\Delta\left(L_{S}(W),L_{\mu}(W)\right)\right],$$ we arrive at the following PAC-Bayesian bound (see, e.g., \cite[Lemma 3]{maurer2004PACBayes}, \cite[Theorem 2.1]{germain2009Paclinear}, \cite[Equation 4]{rivasplatapac}):
\begin{align}\label{eq:TongZhangtempMaurer}
\Pr_{S\sim\mu^{\otimes{n}}}\Bigg( 
\Delta\left(\E_{\post}[L_S(W)],\E_{\post} [L_{\mu}(W)])\right) \le \frac{1}{n} \Big[&D(\post\|\prior)+\ln{\frac{1}{\delta}}\notag\\&+\ln \E_{\prior} \E_{S'\sim \mu^{\otimes n}} \,e^{n\Delta\left(L_{S'}(W),L_{\mu}(W)\right)}\Big]
\Bigg)\ge 1-\delta.
\end{align}
For $x,y\in[0,1]$, the binary KL divergence is $\mathrm{kl}(y\|x)=y\ln\tfrac{y}{x}+(1-y)\ln \tfrac{1-y}{1-x}$. The PAC-Bayes-KL-inequality \cite{seeger2002pac,maurer2004PACBayes} comes about by upper-bounding the log-exponential-moment term involving the true risk in the right hand side of the bound in \eqref{eq:TongZhangtempMaurer}: For $\Delta(y,x) = \mathrm{kl}(y \| x)$, Maurer \cite{maurer2004PACBayes} showed that for $n\ge 8$,  $\E_{\prior}\E_{S'}[e^{n\Delta\left(L_{S'}(W),L_{\mu}(W)\right)}]\le 2\sqrt{n}$, when we have
\begin{align}\label{eq:TongZhangtempMaurerkl}
\Pr_{S\sim\mu^{\otimes{n}}}\Bigg( 
\mathrm{kl}\left(\E_{\post}[L_S(W)],\E_{\post} [L_{\mu}(W)])\right) \le \frac{1}{n} \left[D(\post\|\prior)+\ln{\frac{2\sqrt{n}}{\delta}}\right]
\Bigg)\ge 1-\delta.
\end{align}
\eqref{eq:TongZhangtempMaurerkl} can be interpreted as a ``non-parametric'' version of McAllester's linear PAC-Bayes bound \eqref{eq:McAllesterLinear} that is uniform in $\beta$ at the cost of a $O\left(\tfrac{\ln \sqrt{n}}{n}\right)$ term. 

Letting $\Delta(y,x) = 2(y-x)^2$ in \eqref{eq:TongZhangtempMaurer} leads to the bound in \cite{mcallester1999some}, while letting $\Delta(y,x) = (y-x)^2/(2x)$ leads to that in~\cite{rivasplataThiemannPACBackprop}. 

Under a sub-gamma loss assumption, the bounds in either \eqref{eq:TongZhangtempAlquir} or \eqref{eq:TongZhangtemp2} lead to (see Corollary \ref{thm:PACBayesSubGamma}):
\begin{align}\label{eq:TongZhangsubGamma}
\Pr_{S\sim\mu^{\otimes{n}}}\Bigg( 
\E_{\post} [L_{\mu}(W)] \leq 
\E_{\post} [L_{S}(W)] + \frac{1}{n} \left[D(\post\|\prior)+\ln{\frac{1}{\delta}}\right] + \frac{\sigma^2}{2(1-c)}
\Bigg)\ge 1-\delta.
\end{align}

\subsection{Proofs for Section \ref{sec:DPPriors}}
The proof of Proposition~\ref{thm:PACDPPrior} follows closely that of \cite[Theorem 4.2]{dziugaitediffprivateprior}.
\begin{proof}[Proof of Proposition~\ref{thm:PACDPPrior}]
	For every $\prior\in\cM(\cW)$, let
	\begin{align*}
	F(\prior) = \left\{S'\in\cZ^n\colon \exists\,\post\in\cM(\cW),\, \E_{\post} [M_{\beta}(W)] \ge \E_{\post}[L_{S'}(W)] + \frac{1}{n\beta}\left(D(\post\|\prior)+\ln\frac{1}{\delta'}\right)\right\}.
	\end{align*}
	By Theorem \ref{thm:TongZhangIRM}, we have $\Pr_{S'\sim\mu^{\otimes{n}}}\big(S'\in F(\prior)\big)\le \delta'$. From \eqref{eq:MaxInfoDependence}, we have
	\begin{align*}
	\Pr_{S\sim\mu^{\otimes{n}}}\big(S\in F(\prior^0(S))\big)&\le e^{I_{\infty, \mu}^{\alpha}(Q^0, n)}\cdot \Pr_{(S,S')\sim\mu^{\otimes{2n}}}\big(S'\in F(\prior^0(S))\big)+\alpha\le e^{I_{\infty, \mu}^{\alpha}(Q^0, n)}\cdot\delta'+\alpha.
	\end{align*}
	Letting $\delta:=e^{I_{\infty, \mu}^{\alpha}(Q^0, n)}\cdot\delta'+\alpha$, for $\alpha\in(0,\delta)$ we have, 
	\begin{align*}
	\Pr_{S\sim\mu^{\otimes{n}}}\Bigg(\exists\,\post\in\cM(\cW),\, \E_{\post} [M_{\beta}(W)] &\ge \E_{\post}[L_{S}(W)] \\&+ \frac{1}{n\beta}\left(D(\post\|\prior^0(S))+\ln\frac{1}{\delta-\alpha}+I_{\infty, \mu}^{\alpha}(Q^0, n)\right)\Bigg)\le \delta.
	\end{align*}
	The proof is complete by replacing $I_{\infty, \mu}^{\alpha}(Q^0, n)$ with the bound in Theorem~\ref{thm:max-infoDP}, and choosing $\alpha=\tfrac{\delta}{2}$.
\end{proof}

\subsection{Proofs for Section \ref{sec:PAC-Bayes-SGD}}
We show how to optimize the bound in Corollary \ref{thm:pacbayesCatoniDiff} w.r.t. the parameters $\beta$ and $\lambda$.

First, note that the bound in Corollary \ref{thm:pacbayesCatoniDiff} holds uniformly for all $\beta>1$ at an additional cost arising from a union bound argument \cite[Theorem 1.2.7]{catonibook}:
For $\alpha>1$, 
\begin{align}\label{eq:catoniDD}
\E_{\post}[L_\mu(f_W)] \le \inf_{\beta>1} \Phi_{\beta}^{-1}\Big\{
\E_{\post}[L_S(f_W)] + \frac{\alpha}{n\beta} \Big[ D(\post\|\prior) +\ln\frac{1}{\delta} + 2\ln \Big(\frac{\ln \alpha^2\beta n}{\ln \alpha}\Big)\Big]
\Big\}.
\end{align}

Second, we select $\lambda$ before the draw of the training sample from a finite grid of possible values: Following \cite{langfordNotBounding,dziugaitenonvacuous}, let $\lambda =c e^{-j/b}$ for some $j\in\mathbb{N}$ and fixed $b\in\mathbb{N}$, $c\in (0,1)$, where $b$ and $c$ control, resp., the resolution and size of the grid. If \eqref{eq:catoniDD} holds for each $j\in\mathbb{N}$ with probability of at least $1-\tfrac{6\delta}{\pi^2 j^2}$, then by the union bound, it holds for all $j\in\mathbb{N}$ simultaneously with probability of at least $1-\delta$, since $\sum_{j=1}^{\infty} \tfrac{6}{\pi^2 j^2} = 1$. Solving for $j$ in terms of $\lambda$, we have 
\begin{align*}
\E_{\post}[L_\mu(f_W)] \le 
\inf_{\beta>1,\,\lambda\in (0,c)} 
\Phi_{\beta}^{-1}\Bigg\{
&\E_{\post}[L_S(f_W)]\\&+ \frac{\alpha}{n\beta} \left[ D(\post\|\prior) 
+\ln\left(\frac{\pi^2 b^2}{6\delta}\left(\ln\frac{c}{\lambda}\right)^2\right)
+ 2\ln \Big(\frac{\ln \alpha^2\beta n}{\ln \alpha}\Big)\right]
\Bigg\}.
\end{align*}

Finally, we account for the cost of using a Monte Carlo estimate of the empirical risk, 
$\hat{L}_S(f_W)=\tfrac{1}{m}\sum_{i=1}^{m} L_S(f_{W_i})$, where $W_i\overset{\text{i.i.d.}}{\sim}{\post}$. 
By an application of the Chernoff bound \cite[Theorem 2.5]{langfordNotBounding} 
and Pinsker's inequality, for any $\delta'\in(0,1)$, we have with probability of at least $1-\delta'$, $\E_{\post}[{L}_S(f_W)]\le \hat{L}_S(f_W) +\sqrt{\tfrac{1}{2m}\ln\tfrac{2}{\delta'}}$. 

By another application of the union bound, Corollary \ref{thm:pacbayesCatoniDiff} finally takes the form: For any $\delta,\delta'\in(0,1)$, fixed $\alpha>1$, $c\in (0,1)$, $b\in\mathbb{N}$, and $m,n\in\mathbb{N}$, with probability of at least $1-\delta-\delta'$ over a draw of $S\sim\mu^{\otimes n}$ and $W\sim (\post)^{\otimes m}$, 
\begin{align*}
\E_{\post}[L_{\mu}(f_W)] \le \inf_{\post\in\cG,\beta>1,\lambda\in (0,c)}&\Phi_{\beta}^{-1}\Big\{\hat{L}_S(f_W)+ \frac{\alpha}{n\beta}D(\post\|\prior)\notag+R(\lambda,\beta;\delta,\delta')\Big\},
\end{align*}
where $R=\frac{2\alpha}{n\beta}\ln \left[\frac{\ln \alpha^2\beta n}{\ln \alpha}\right]+\frac{\alpha}{n\beta}\ln\left[\tfrac{\pi^2 b^2}{6\delta}\left(\ln\tfrac{c}{\lambda}\right)^2\right]+\sqrt{\tfrac{1}{2m}\ln\tfrac{2}{\delta'}}$.

\subsection{Proofs for Section \ref{sec:PACOccam}}
\begin{proof}[Proof of Lemma \ref{lemma:GibbsLaplaceApprox}]
	Letting $\theta = w - w_{\post}$, and $\post' = \post - w_{\post}$, note that $\theta^\top H\theta = \tr(\theta^\top H\theta) = \tr(H\theta\theta^\top)$. Hence 
	\[\E_{\post'}[\tfrac{1}{2}\theta^\top H\theta] = \E_{\post'}[\tfrac{1}{2}\tr(H\theta\theta^\top)] = \tfrac{1}{2}\tr(H\E_{\post'}[\theta\theta^\top])=\tfrac{1}{2}\tr(H\Sigma_{\post}).
	\]
	For $\prior \sim \cN(w_{\prior},\Sigma_{\prior})$ and $\post \sim \cN(w_{\post},\Sigma_{\post})$, we have
	\begin{align*}
	\E_{\post'}&[\tfrac{1}{2}\theta^\top H\theta]+{(n\beta)}^{-1} D(\post\|\prior) \\&= \tfrac{1}{2}\tr(H\Sigma_{\post})+{(n\beta)}^{-1} D(\post\|\prior) 
	\\&= \frac{\tr(H\Sigma_{\post})}{2} + \frac{{(n\beta)}^{-1}}{2}\left(\ln \frac{\det\Sigma_{\prior}}{\det \Sigma_{\post}}
	+\tr\left(\Sigma_{\prior}^{-1} \Sigma_{\post}\right)-k
	+\left(w_{\prior}-w_{\post}\right)^{\top} \Sigma_{\prior}^{-1}\left(w_{\prior}-w_{\post}\right)\right).
	\end{align*}
	The derivative of the RHS w.r.t. $\Sigma_{\post}$ is 
	$\frac{1}{2}\left[{H} - {{(n\beta)}^{-1}}\Sigma_{\post}^{-1} + {{(n\beta)}^{-1}}\Sigma_{\prior}^{-1}\right]^\top$,
	where we have used the fact that $\nabla_A \tr(AB) = B^\top, \text{ and } \nabla_A \ln\det(A) = (A^{-1})^\top$. 
	Setting the derivative to zero and $\Sigma_{\prior}={\lambda}^{-1}I_k$ yields the result.
\end{proof}

\begin{proof}[Proof of Proposition \ref{thm:OccamPAC}]
	The proof follows from Theorem \ref{thm:TongZhangIRM}, and the fact that for $\prior = \cN(w_{\prior},{\lambda}^{-1}I_k)$, $\post = \cN(w_{\post},H_{\lambda}^{-1})$ such that $\lambda_i\ge \lambda>0$ for all $i$, we have
	\begin{align*}
	D(\post\|\prior)=
	\frac{1}{2}\left(
	\lambda\|w_{\prior}-w_{\post}\|^2
	+\sum_{i=1}^{k}\ln \frac{\lambda_i}{\lambda}
	+\sum_{i=1}^{k}\left(\frac{\lambda}{\lambda_i}-1\right)
	\right) 
	\le
	\frac{1}{2}\left(\lambda\|w_{\prior}-w_{\post}\|^2
	+\sum_{i=1}^{k}\ln \frac{\lambda_i}{\lambda}
	\right).
	\end{align*}
\end{proof}

\end{document}